\documentclass{article}
\usepackage{authblk}

\usepackage{hyperref}       
\usepackage{url}            
\usepackage{booktabs}       
\usepackage{graphicx}
\graphicspath{ {./images/} }
\usepackage{amsmath}
\usepackage{amsthm}
\usepackage{amsfonts}
\usepackage{amssymb}
\usepackage{mathtools}
\usepackage[usenames,dvipsnames]{xcolor}
\usepackage{soul}
\usepackage[makeroom]{cancel}
\usepackage{tikz}
\usepackage{bm}
\usepackage{tocloft}
\definecolor{toccolor}{RGB}{139,0,139}

\usepackage{mystyle}
\usepackage{geometry}
\usepackage{setspace}

\newcommand\Ccancel[2][red]{\renewcommand\CancelColor{\color{#1}}\cancel{#2}}
\newcommand*\circled[1]{\tikz[baseline=(char.base)]{
            \node[shape=circle,draw,inner sep=2pt] (char) {#1};}}

\title{Diffusion Models: A Mathematical Introduction}
\author[1]{Sepehr Maleki\thanks{\href{mailto:smaleki@lincoln.ac.uk}{smaleki@lincoln.ac.uk}}}
\author[2]{Negar Pourmoazemi\thanks{\href{mailto:negar.pourmoazemi@thetrainline.com}{negar.pourmoazemi@thetrainline.com}}}
\affil[1]{Lincoln AI Lab, University of Lincoln, Lincoln, The UK}
\affil[2]{Trainline, London, The UK}

\date{}

\begin{document}
\maketitle

\begin{abstract}
\noindent We present a concise, self-contained derivation of diffusion-based generative models. Starting from basic properties of Gaussian distributions (densities, quadratic expectations, re-parameterisation, products, and KL divergences), we construct denoising diffusion probabilistic models from first principles. This includes the forward noising process, its closed-form marginals, the exact discrete reverse posterior, and the related variational bound. This bound simplifies to the standard noise-prediction goal used in practice. We then discuss likelihood estimation and accelerated sampling, covering DDIM, adversarially learned reverse dynamics (DDGAN), and multi-scale variants such as nested and latent diffusion, with Stable Diffusion as a canonical example. A continuous-time formulation follows, in which we derive the probability-flow ODE from the diffusion SDE via the continuity and Fokker-Planck equations, introduce flow matching, and show how rectified flows recover DDIM up to a time re-parameterisation. Finally, we treat guided diffusion, interpreting classifier guidance as a posterior score correction and classifier-free guidance as a principled interpolation between conditional and unconditional scores. Throughout, the focus is on transparent algebra, explicit intermediate steps, and consistent notation, so that readers can both follow the theory and implement the corresponding algorithms in practice.
\end{abstract}

\tableofcontents
\newpage

\section{Introduction}
\label{sec:introduction}

Denoising diffusion models have become a central paradigm in deep generative modelling. They construct a tractable \emph{forward} (diffusion) process that gradually corrupts data with Gaussian noise and then learn a \emph{reverse} (generative) process that denoises step by step. This simple idea—noise in, data out—scales remarkably well, delivering state-of-the-art results across images, text-to-image synthesis, audio, and beyond.\medskip

The mathematics, however, can feel opaque at first glance. One must keep track of coupled Markov chains and their Gaussian conditionals, derive a variational (ELBO) objective, choose and interpret noise schedules, and understand how conditioning mechanisms (classifier-based and classifier-free guidance) alter the reverse dynamics. A continuous-time view further links the discrete model to stochastic differential equations (SDEs) and the associated probability–flow ordinary differential equation (ODE), while recent work reframes sampling as learning a velocity field via flow matching. This tutorial develops these ingredients \emph{from first principles} with a consistent notation and complete intermediate steps.\medskip

We begin with Gaussian preliminaries (densities, affine transforms, products, KL divergence) to fix identities used throughout. We then build Denoising Diffusion Probabilistic Models (DDPMs) carefully: the forward and reverse chains, the exact DDPM posterior, and the variational loss, showing how the standard training objective arises under the $\epsilon$-prediction parameterisation. Along the way we make explicit the relationships among $\epsilon$-, $\hat{\mathbf{x}}_0$-, and velocity parameterisations, clarify the reverse mean factor $\beta_t/\sqrt{1-\bar{\alpha}_t}$, and give the role of the posterior variance $\tilde{\beta}_t$. Next, we connect discrete diffusion to the score-based and probability–flow viewpoints and explain how these perspectives motivate accelerated samplers such as DDIM. We then profile practical accelerations and latent-space methods (e.g., Stable Diffusion), and treat guided diffusion in detail—deriving classifier guidance as a posterior-score modification and classifier-free guidance as a principled conditional–unconditional blend, with remarks on time-varying guidance schedules, numerical stability, and common failure modes. Finally, we cover flow matching, including conditional and marginal objectives and straight-line (rectified) flows, with concise proofs to keep the development accessible yet rigorous.\medskip

The tutorial targets researchers and students seeking a working, mathematical understanding of diffusion-based generators. Prior exposure to probability, basic linear algebra for Gaussians, and deep learning fundamentals is helpful but not strictly required; the text is self-contained and proofs are kept elementary whenever possible. Our goal is to replace folklore and code-first heuristics with transparent algebra and immediately usable formulas.\medskip

\subsection*{Reader’s Roadmap}
If you are new to the topic, start with the Gaussian preliminaries and the DDPM derivation, which establish the notation and the variational objective. Readers primarily interested in faster sampling may jump to the acceleration material (DDIM and related methods) and the latent-space section on Stable Diffusion. Those focused on controllability should see the guided diffusion section (classifier and classifier-free guidance, schedules, and pitfalls). The flow matching section unifies the ODE perspective and offers an alternative training route.

\subsection*{Notation}

All random vectors live in $\mathbb{R}^d$ unless stated. We use bold symbols for vectors/matrices (e.g., $\mathbf{x}, \mathbf{z}, \mathbf{\Sigma}$); the identity is $\mathbf{I}_d$ (we simply write $\mathbf{I}$ when the dimension is clear). The Euclidean norm is $\|\cdot\|_2$, inner products are $\langle \cdot,\cdot\rangle$, the trace is $\mathrm{tr}(\cdot)$, and determinants are $\det(\cdot)$. A Gaussian with mean $\boldsymbol{\mu}$ and covariance $\mathbf{\Sigma}$ is $\mathcal{N}(\boldsymbol{\mu},\mathbf{\Sigma})$. Expectations are $\mathbb{E}[\cdot]$, variances $\mathrm{Var}[\cdot]$, and $\mathrm{KL}(P\|Q)$ denotes Kullback–Leibler divergence. We write $\mathrm{law}(\mathbf{X})$ for the distribution of a random vector $\mathbf{X}$.

Data are $\mathbf{x}_0\!\sim p_{\text{data}}$; corrupted states are $\mathbf{x}_t$ for discrete $t\in\{1,\dots,T\}$. The model distribution is $p_\theta$ and the forward (noising) distribution is $q$. In particular, the forward chain is
\[
q(\mathbf{x}_{1:T}\mid \mathbf{x}_0)=\prod_{t=1}^T q(\mathbf{x}_t\mid \mathbf{x}_{t-1}),
\qquad
q(\mathbf{x}_t\mid \mathbf{x}_{t-1})
= \mathcal{N}\!\big(\sqrt{\alpha_t}\,\mathbf{x}_{t-1},\,\beta_t\,\mathbf{I}\big),
\]
with schedule parameters $\alpha_t\in(0,1)$ and $\beta_t:=1-\alpha_t$. We use the cumulative product
\[
\bar{\alpha}_t := \prod_{i=1}^t \alpha_i,\qquad
\mathrm{SNR}_t := \frac{\bar{\alpha}_t}{1-\bar{\alpha}_t},\qquad
\ell_t := \log \mathrm{SNR}_t.
\]
The closed form marginal and reparameterisation are
\[
q(\mathbf{x}_t\mid \mathbf{x}_0)=\mathcal{N}\!\big(\sqrt{\bar{\alpha}_t}\,\mathbf{x}_0,\,(1-\bar{\alpha}_t)\mathbf{I}\big),
\qquad
\mathbf{x}_t=\sqrt{\bar{\alpha}_t}\,\mathbf{x}_0+\sqrt{1-\bar{\alpha}_t}\,\boldsymbol{\epsilon},\ \boldsymbol{\epsilon}\sim\mathcal{N}(\mathbf{0},\mathbf{I}).
\]

The reverse (denoising) conditionals are
\[
p_\theta(\mathbf{x}_{t-1}\mid \mathbf{x}_t)
=\mathcal{N}\!\big(\mu_\theta(\mathbf{x}_t,t),\,\sigma_t^2\mathbf{I}\big),
\quad
\mu_\theta(\mathbf{x}_t,t)=\frac{1}{\sqrt{\alpha_t}}\!\left(\mathbf{x}_t-\frac{\beta_t}{\sqrt{1-\bar{\alpha}_t}}\ \hat{\boldsymbol{\epsilon}}_\theta(\mathbf{x}_t,t)\right),
\]
where $\hat{\boldsymbol{\epsilon}}_\theta$ is the learned noise predictor and we use the posterior variance
\[
\tilde{\beta}_t := \frac{1-\bar{\alpha}_{t-1}}{1-\bar{\alpha}_t}\,\beta_t,
\qquad
\sigma_t^2\in\{\tilde{\beta}_t,\ 0\}\ \text{for DDPM/DDIM respectively.}
\]
We write $\hat{\mathbf{x}}_0(\mathbf{x}_t,t)$ for an $x_0$-prediction when used, and $s_t(\mathbf{x}) := \nabla_{\mathbf{x}}\log p_t(\mathbf{x})$ for the (true or learned) score at time $t$.

For continuous-time notation, $\mathbf{X}_t$ denotes the state at $t\in[0,T]$, Brownian motion is $\mathbf{W}_t$, and the diffusion SDE is
\[
\mathrm{d}\mathbf{X}_t=\mathbf{f}(\mathbf{X}_t,t)\,\mathrm{d}t+g(t)\,\mathrm{d}\mathbf{W}_t,
\]
with probability–flow ODE velocity $\mathbf{v}(\mathbf{x},t)$ producing the same time marginals. In flow-matching sections we use $\mathbf{v}_\theta(\mathbf{x},t)$ for a learned velocity, and $\psi_t(\cdot)$ for an interpolant between endpoints.

Conditioning variables are $\mathbf{c}$; classifier-free guidance uses a scale $\lambda\ge 0$ and the blended predictor
\[
\hat{\boldsymbol{\epsilon}}_\lambda
= \hat{\boldsymbol{\epsilon}}_{\text{uncond}}
+\lambda\big(\hat{\boldsymbol{\epsilon}}_{\text{cond}}-\hat{\boldsymbol{\epsilon}}_{\text{uncond}}\big),
\]
with the score-form $s_t^{(\lambda)}(\mathbf{x})=-(1/\sqrt{1-\bar{\alpha}_t})\,\hat{\boldsymbol{\epsilon}}_\lambda(\mathbf{x},t,\mathbf{c})$.
When discussing latent-space methods, latents are denoted $\mathbf{z}_t$ and are treated as vectors inside Gaussian densities (implicit flattening).
\medskip

\section{Preliminaries}

\subsection{Density of an isotropic Gaussian}

Let $\x \in \mathbb{R}^d$ and consider the $d$-dimensional Gaussian $\mathcal{N}(\boldsymbol{\mu}, \mathbf{\Sigma})$. Its Probability Density Function (PDF) is
\begin{equation}
\label{eq:normPDF}
p(\x) \;=\; \frac{1}{(2\pi)^{d/2}\,\mathrm{det}(\mathbf{\Sigma})^{1/2}}
\exp\!\Big(-\tfrac{1}{2}\,(\x-\boldsymbol{\mu})^\top \mathbf{\Sigma}^{-1} (\x-\boldsymbol{\mu})\Big),
\end{equation}
where $\boldsymbol{\mu}\in\mathbb{R}^d$ and $\mathbf{\Sigma}\in\mathbb{R}^{d\times d}$ is symmetric positive-definite.\medskip

An \emph{isotropic} Gaussian is the special case with covariance $\mathbf{\Sigma}=\sigma^2 \mathbf{I}$ for some $\sigma>0$, i.e., the variance is the same in every direction (equivalently, the distribution is rotationally invariant). In this case,
\[
\mathrm{det}(\mathbf{\Sigma})=\sigma^{2d},
\qquad
\mathbf{\Sigma}^{-1}=\frac{1}{\sigma^2}\,\mathbf{I},
\]
and the quadratic (Mahalanobis) term in \eqref{eq:normPDF} reduces to
\[
(\x-\boldsymbol{\mu})^\top \mathbf{\Sigma}^{-1} (\x-\boldsymbol{\mu})
= \frac{1}{\sigma^2}\,(\x-\boldsymbol{\mu})^\top(\x-\boldsymbol{\mu})
= \frac{1}{\sigma^2}\,||\x-\boldsymbol{\mu}||_2^2.
\]

where $\|\x-\boldsymbol{\mu}\|_2^2$ is the squared $\ell_2$-norm (Euclidean distance). Therefore,
\begin{equation}
\label{eq:isotropicPDF}
p(\x) \;=\; \frac{1}{(2\pi\sigma^2)^{d/2}}\,
\exp\!\Big(-\frac{1}{2\sigma^2}\,\|\x-\boldsymbol{\mu}\|_2^2\Big).
\end{equation}

\medskip

\subsection{Expectation of a quadratic form}

We will often encounter quadratic expressions of the form
$(\x-\boldsymbol{\mu}_p)^\top \mathbf{A}\,(\x-\boldsymbol{\mu}_p)$
under a Gaussian distribution. This subsection gives the identities we will use later (e.g., in the Gaussian KL derivation) and provides short proofs.

\begin{definition}[Trace]
For a matrix $\mathbf{A}\in\mathbb{R}^{d\times d}$, the \emph{trace} is the sum of its diagonal entries:
\[
\mathrm{tr}(\mathbf{A}) \;=\; \sum_{i=1}^d A_{ii}.
\]
\end{definition}

\begin{lemma}[Trace identities]\label{lem:trace-identities}
Let $\mathbf{A},\mathbf{B}\in\mathbb{R}^{d\times d}$ and $\mathbf{u},\mathbf{v}\in\mathbb{R}^d$. Then:
\begin{enumerate}\itemsep4pt
  \item \textbf{Linearity:} For all scalars $\alpha,\beta\in\mathbb{R}$,
  \[
  \mathrm{tr}(\alpha \mathbf{A} + \beta \mathbf{B}) \;=\; \alpha\,\mathrm{tr}(\mathbf{A}) + \beta\,\mathrm{tr}(\mathbf{B}).
  \]
  \item \textbf{Cyclicity:} Whenever the products are defined,
  \[
  \mathrm{tr}(\mathbf{A}\mathbf{B}) \;=\; \mathrm{tr}(\mathbf{B}\mathbf{A}).
  \]
  \item \textbf{Rank-one case:}
  \[
  \mathrm{tr}(\mathbf{u}\mathbf{v}^\top) \;=\; \mathbf{v}^\top \mathbf{u}.
  \]
\end{enumerate}
\end{lemma}

\begin{lemma}[Trace trick]\label{lem:trace-trick}
For any $\mathbf{A}\in\mathbb{R}^{d\times d}$ and $\mathbf{u},\mathbf{v}\in\mathbb{R}^d$,
\[
\mathbf{v}^\top \mathbf{A}\,\mathbf{u} \;=\; \mathrm{tr}\!\big(\mathbf{A}\,\mathbf{u}\mathbf{v}^\top\big).
\]
\begin{proof}
We start from the right hand side. By Lemma~\ref{lem:trace-identities} (ii) and (iii):

\[
\mathrm{tr}(\mathbf{A}\mathbf{u}\mathbf{v}^\top)=\mathrm{tr}(\mathbf{u}\mathbf{v}^\top \mathbf{A})
= \mathrm{tr}(\mathbf{u} (\mathbf{A}^\top\mathbf{v})^\top)
 =\mathbf{v}^\top \mathbf{A}\mathbf{u}\;.
\]
\end{proof}
\end{lemma}

\begin{proposition}[Expectation of a centered quadratic form]\label{prop:quad-centered}
Let $p$ denote the Gaussian $\mathcal{N}(\boldsymbol{\mu}_p,\mathbf{\Sigma}_p)$ on $\mathbb{R}^d$, and let $\x\sim p$. For any $\mathbf{A}\in\mathbb{R}^{d\times d}$,
\begin{equation}\label{eq:quad-centered}
\mathbb{E}_{p}\!\left[(\x-\boldsymbol{\mu}_p)^\top \mathbf{A}\,(\x-\boldsymbol{\mu}_p)\right]
\;=\; \mathrm{tr}\!\big(\mathbf{A}\,\mathbf{\Sigma}_p\big).
\end{equation}
\begin{proof}
Let $\mathbf{z}\coloneqq \x-\boldsymbol{\mu}_p\sim\mathcal{N}(\mathbf{0},\mathbf{\Sigma}_p)$.
By Lemma~\ref{lem:trace-trick} and linearity of trace and expectation,
\[
\mathbb{E}_p[\mathbf{z}^\top \mathbf{A}\mathbf{z}]
= \mathbb{E}_p\big[\mathrm{tr}(\mathbf{A}\,\mathbf{z}\mathbf{z}^\top)\big]
= \mathrm{tr}\!\big(\mathbf{A}\,\mathbb{E}_p[\mathbf{z}\mathbf{z}^\top]\big)
= \mathrm{tr}(\mathbf{A}\,\mathbf{\Sigma}_p).
\]
\end{proof}
\end{proposition}

\subsection{Reparameterisation for Gaussian Sampling}
\label{subsec:gaussian_reparam}

We want a simple, differentiable way to sample from a multivariate Gaussian $\mathcal{N}(\boldsymbol{\mu}, \mathbf{\Sigma}) \subset \mathbb{R}^d$. The key idea is to express a sample as an affine function of a standard normal.

\begin{definition}[Affine re-parameterisation]
Let $\mathbf{z} \sim \mathcal{N}(\mathbf{0}, \mathbf{I}_d)$. For $\boldsymbol{\mu} \in \mathbb{R}^d$ and $\mathbf{A} \in \mathbb{R}^{d \times d}$, define
\[
\mathbf{x} := \boldsymbol{\mu} + \mathbf{A}\mathbf{z}.
\]
\end{definition}

\begin{proposition}[Distribution of an affine transform of a standard Gaussian]
\label{prop:affine_gaussian_easy}
If $\mathbf{z} \sim \mathcal{N}(\mathbf{0}, \mathbf{I}_d)$, then
\[
\mathbf{x}=\boldsymbol{\mu}+\mathbf{A}\mathbf{z}
\;\sim\;
\mathcal{N}\!\big(\boldsymbol{\mu},\, \mathbf{A}\mathbf{A}^\top\big).
\]
\end{proposition}

\begin{proof}
Write $\mathbf{x}=\boldsymbol{\mu}+\mathbf{A}\mathbf{z}$ with $\mathbb{E}[\mathbf{z}]=\mathbf{0}$ and $\mathrm{Cov}[\mathbf{z}]=\mathbf{I}_d$ (i.e., $\mathbf{z}$ is a standard normal vector).
The mean is immediate:
\[
\mathbb{E}[\mathbf{x}]
= \boldsymbol{\mu} + \mathbf{A}\,\mathbb{E}[\mathbf{z}]
= \boldsymbol{\mu}.
\]
For the covariance, use the definition $\mathrm{Cov}[\mathbf{x}]=\mathbb{E}\!\big[(\mathbf{x}-\mathbb{E}[\mathbf{x}])(\mathbf{x}-\mathbb{E}[\mathbf{x}])^\top\big]$:
\begin{align*}
\mathrm{Cov}[\mathbf{x}]
&= \mathbb{E}\!\big[(\boldsymbol{\mu}+\mathbf{A}\mathbf{z}-\boldsymbol{\mu})(\boldsymbol{\mu}+\mathbf{A}\mathbf{z}-\boldsymbol{\mu})^\top\big] \\
&= \mathbb{E}\!\big[(\mathbf{A}\mathbf{z})(\mathbf{A}\mathbf{z})^\top\big] \\
&= \mathbb{E}\!\big[\mathbf{A}\,\mathbf{z}\mathbf{z}^\top\,\mathbf{A}^\top\big] \\
&= \mathbf{A}\,\mathbb{E}[\mathbf{z}\mathbf{z}^\top]\,\mathbf{A}^\top
\quad\text{(linearity of $\mathbb{E}$; $\mathbf{A}$ and $\mathbf{A}^\top$ are constants)} \\
&= \mathbf{A}\,\mathbf{I}_d\,\mathbf{A}^\top
\quad\text{since }\mathbb{E}[\mathbf{z}\mathbf{z}^\top]=\mathrm{Cov}[\mathbf{z}]=\mathbf{I}_d \\
&= \mathbf{A}\mathbf{A}^\top.
\end{align*}
To see $\mathbb{E}[\mathbf{z}\mathbf{z}^\top]=\mathbf{I}_d$ explicitly, write $\mathbf{z}=(z_1,\dots,z_d)^\top$ with each $z_i\sim\mathcal{N}(0,1)$ and uncorrelated. Then
\[
\big[\mathbb{E}[\mathbf{z}\mathbf{z}^\top]\big]_{ij}=\mathbb{E}[z_i z_j]
=\begin{cases}
0, & i\neq j,\\
1, & i=j,
\end{cases}
\]
so the matrix is exactly $\mathbf{I}_d$.
\end{proof}

\begin{corollary}[Sampling $\mathcal{N}(\boldsymbol{\mu}, \mathbf{\Sigma})$]
\label{cor:full_cov_sampling_easy}
If $\mathbf{\Sigma}$ is symmetric positive semidefinite and $\mathbf{A}$ satisfies $\mathbf{A}\mathbf{A}^\top=\mathbf{\Sigma}$ (e.g., Cholesky when $\mathbf{\Sigma}\succ\mathbf{0}$), then with $\mathbf{z}\sim\mathcal{N}(\mathbf{0},\mathbf{I}_d)$,
\[
\mathbf{x}=\boldsymbol{\mu}+\mathbf{A}\mathbf{z} \;\sim\; \mathcal{N}(\boldsymbol{\mu}, \mathbf{\Sigma}).
\]
\end{corollary}

\begin{corollary}[Isotropic case]
\label{cor:isotropic_easy}
For $\sigma>0$, with $\mathbf{z}\sim\mathcal{N}(\mathbf{0},\mathbf{I}_d)$,
\[
\mathbf{x}=\boldsymbol{\mu}+\sigma\,\mathbf{z}
\;\sim\;
\mathcal{N}(\boldsymbol{\mu},\, \sigma^2 \mathbf{I}_d).
\]
\end{corollary}

\begin{proposition}[Reparameterisation identity for expectations]
\label{prop:reparam_identity_easy}
Let $f:\mathbb{R}^d\to\mathbb{R}$ be integrable under $\mathcal{N}(\boldsymbol{\mu}, \mathbf{\Sigma})$. If $\mathbf{\Sigma}=\mathbf{A}\mathbf{A}^\top$ and $\mathbf{z}\sim\mathcal{N}(\mathbf{0},\mathbf{I}_d)$, then
\[
\mathbb{E}_{\mathbf{x}\sim\mathcal{N}(\boldsymbol{\mu}, \mathbf{\Sigma})}[f(\mathbf{x})]
=
\mathbb{E}_{\mathbf{z}\sim\mathcal{N}(\mathbf{0},\mathbf{I}_d)}\!\big[f(\boldsymbol{\mu}+\mathbf{A}\mathbf{z})\big].
\]
\end{proposition}

\subsection{Product of two Gaussian distributions}

We will often need to combine two Gaussian factors over the same variable (e.g., when merging information from two sources). The key fact is that the product is again (proportional to) a Gaussian. To read off the resulting mean and covariance cleanly, we first note a standard “completing the square” identity.

\begin{lemma}[Completing the square]\label{lem:complete-square}
Let $\mathbf{\Lambda}\in\mathbb{R}^{d\times d}$ be symmetric positive-definite and let $\boldsymbol{\eta}\in\mathbb{R}^d$. Then, for all $\x\in\mathbb{R}^d$,
\[
-\tfrac{1}{2}\,\x^\top \mathbf{\Lambda}\,\x + \x^\top \boldsymbol{\eta}
\;=\;
-\tfrac{1}{2}\,(\x-\boldsymbol{\mu})^\top \mathbf{\Lambda}\,(\x-\boldsymbol{\mu})
\;+\; \tfrac{1}{2}\,\boldsymbol{\mu}^\top \mathbf{\Lambda}\,\boldsymbol{\mu},
\qquad
\text{where }\ \boldsymbol{\mu}=\mathbf{\Lambda}^{-1}\boldsymbol{\eta}.
\]
\emph{Proof.}
Expand the right-hand side:
$-\tfrac{1}{2}(\x^\top \mathbf{\Lambda}\x - 2\,\x^\top \mathbf{\Lambda}\boldsymbol{\mu}+\boldsymbol{\mu}^\top \mathbf{\Lambda}\boldsymbol{\mu})
+\tfrac{1}{2}\boldsymbol{\mu}^\top \mathbf{\Lambda}\boldsymbol{\mu}
= -\tfrac{1}{2}\x^\top \mathbf{\Lambda}\x + \x^\top \mathbf{\Lambda}\boldsymbol{\mu}$.
Since $\boldsymbol{\mu}=\mathbf{\Lambda}^{-1}\boldsymbol{\eta}$, we have $\mathbf{\Lambda}\boldsymbol{\mu}=\boldsymbol{\eta}$, giving the claim.
\hfill$\square$
\end{lemma}

\begin{proposition}[Product of two Gaussians]\label{prop:product-gaussians}
Let $p(\x)=\mathcal{N}(\x;\boldsymbol{\mu}_p,\mathbf{\Sigma}_p)$ and $q(\x)=\mathcal{N}(\x;\boldsymbol{\mu}_q,\mathbf{\Sigma}_q)$ with symmetric positive-definite covariances. Then their pointwise product is proportional to a Gaussian:
\[
p(\x)\,q(\x)\ \propto\ \mathcal{N}(\x;\,\boldsymbol{\mu},\,\mathbf{\Sigma}),
\]
with
\[
\mathbf{\Sigma}=\big(\mathbf{\Sigma}_p^{-1}+\mathbf{\Sigma}_q^{-1}\big)^{-1},
\qquad
\boldsymbol{\mu}=\mathbf{\Sigma}\big(\mathbf{\Sigma}_p^{-1}\boldsymbol{\mu}_p+\mathbf{\Sigma}_q^{-1}\boldsymbol{\mu}_q\big).
\]

\end{proposition}

\begin{proof}
Write each Gaussian in its exponential (``natural'') form, dropping constants that do not depend on $\x$:
\[
p(\x)\ \propto\ \exp\!\Big(-\tfrac{1}{2}\x^\top \mathbf{\Sigma}_p^{-1}\x + \x^\top \mathbf{\Sigma}_p^{-1}\boldsymbol{\mu}_p\Big),
\qquad
q(\x)\ \propto\ \exp\!\Big(-\tfrac{1}{2}\x^\top \mathbf{\Sigma}_q^{-1}\x + \x^\top \mathbf{\Sigma}_q^{-1}\boldsymbol{\mu}_q\Big).
\]
Multiplying and collecting like terms yields
\[
p(\x)\,q(\x)\ \propto\ \exp\!\Big(
-\tfrac{1}{2}\x^\top \underbrace{(\mathbf{\Sigma}_p^{-1}+\mathbf{\Sigma}_q^{-1})}_{\mathbf{\Lambda}} \x
\;+\;
\x^\top \underbrace{(\mathbf{\Sigma}_p^{-1}\boldsymbol{\mu}_p+\mathbf{\Sigma}_q^{-1}\boldsymbol{\mu}_q)}_{\boldsymbol{\eta}}
\Big).
\]
Apply Lemma~\ref{lem:complete-square} with $\mathbf{\Lambda}=\mathbf{\Sigma}_p^{-1}+\mathbf{\Sigma}_q^{-1}$ and $\boldsymbol{\eta}=\mathbf{\Sigma}_p^{-1}\boldsymbol{\mu}_p+\mathbf{\Sigma}_q^{-1}\boldsymbol{\mu}_q$ to rewrite the exponent as
$-\tfrac{1}{2}(\x-\boldsymbol{\mu})^\top \mathbf{\Lambda}(\x-\boldsymbol{\mu}) + \text{const}$,
where $\boldsymbol{\mu}=\mathbf{\Lambda}^{-1}\boldsymbol{\eta}$.
Recognising $\mathbf{\Sigma}=\mathbf{\Lambda}^{-1}$ gives the stated mean and covariance. The omitted constant ensures proper normalisation but does not affect the dependence on $\x$.
\end{proof}

\begin{corollary}[Isotropic special case]\label{cor:product-isotropic}
If $\mathbf{\Sigma}_p=\sigma_p^2\mathbf{I}_d$ and $\mathbf{\Sigma}_q=\sigma_q^2\mathbf{I}_d$, then
\[
\mathbf{\Sigma}=\Big(\tfrac{1}{\sigma_p^2}+\tfrac{1}{\sigma_q^2}\Big)^{-1}\mathbf{I}_d,
\qquad
\boldsymbol{\mu}
=\Big(\tfrac{1}{\sigma_p^2}+\tfrac{1}{\sigma_q^2}\Big)^{-1}
\Big(\tfrac{\boldsymbol{\mu}_p}{\sigma_p^2}+\tfrac{\boldsymbol{\mu}_q}{\sigma_q^2}\Big),
\]
i.e., the precision (inverse variance) adds, and the mean is a precision-weighted average.
\end{corollary}

In the above we write \(p(\x)\,q(\x)\propto \mathcal{N}(\x;\boldsymbol{\mu},\mathbf{\Sigma})\) because the point-wise product of two probability densities is not, in general, a normalised density. The missing factor is a constant that does not depend on \(\x\); dividing by this constant yields the properly normalised Gaussian pdf with parameters \((\boldsymbol{\mu},\mathbf{\Sigma})\) given in the proposition. The same algebra extends to any finite number of Gaussian factors: if \(p_k(\x)=\mathcal{N}(\x;\boldsymbol{\mu}_k,\mathbf{\Sigma}_k)\) for \(k=1,\ldots,n\), then their product is proportional to a Gaussian with precision (inverse covariance) equal to the sum of precisions and mean given by a precision-weighted average,
\[
\mathbf{\Sigma}^{-1}=\sum_{k=1}^n \mathbf{\Sigma}_k^{-1},
\qquad
\boldsymbol{\mu}=\mathbf{\Sigma}\sum_{k=1}^n \mathbf{\Sigma}_k^{-1}\boldsymbol{\mu}_k.
\]
Intuitively: precisions add, and the mean is pulled towards component means in proportion to their precisions.

\subsection{KL divergence between Gaussian distributions}

We now compute the Kullback–Leibler divergence between two multivariate Gaussians
\[
P=\mathcal{N}(\boldsymbol{\mu}_P,\mathbf{\Sigma}_P),
\qquad
Q=\mathcal{N}(\boldsymbol{\mu}_Q,\mathbf{\Sigma}_Q),
\]
both on $\mathbb{R}^d$ with symmetric positive-definite covariances. By definition,
\[
\mathrm{KL}(P\|Q)=\mathbb{E}_{P}\!\left[\log\frac{p(\x)}{q(\x)}\right]
=\mathbb{E}_{P}\big[\log p(\x)-\log q(\x)\big].
\]

\begin{proposition}[KL for Gaussians]\label{prop:kl-gaussians}
With $P$ and $Q$ as above,
\begin{equation}\label{eq:kl-general}
\mathrm{KL}(P\|Q)
=\tfrac{1}{2}\Big(
\mathrm{tr}(\mathbf{\Sigma}_Q^{-1}\mathbf{\Sigma}_P)
+ (\boldsymbol{\mu}_Q-\boldsymbol{\mu}_P)^\top \mathbf{\Sigma}_Q^{-1}(\boldsymbol{\mu}_Q-\boldsymbol{\mu}_P)
- d
+ \log\tfrac{\mathrm{det}(\mathbf{\Sigma}_Q)}{\mathrm{det}(\mathbf{\Sigma}_P)}
\Big).
\end{equation}
\end{proposition}

\begin{proof}
Starting from
\[
p(\x)=\frac{1}{(2\pi)^{d/2}\,\mathrm{det}(\mathbf{\Sigma}_P)^{1/2}}
\exp\!\Big(-\tfrac{1}{2}(\x-\boldsymbol{\mu}_P)^\top \mathbf{\Sigma}_P^{-1}(\x-\boldsymbol{\mu}_P)\Big),
\]
taking logs gives
\[
\log p(\x)=-\tfrac{d}{2}\log(2\pi)-\tfrac{1}{2}\log\mathrm{det}(\mathbf{\Sigma}_P)
-\tfrac{1}{2}(\x-\boldsymbol{\mu}_P)^\top \mathbf{\Sigma}_P^{-1}(\x-\boldsymbol{\mu}_P),
\]
and analogously for $\log q(\x)$ with $(\boldsymbol{\mu}_Q,\mathbf{\Sigma}_Q)$. Thus
\[
\log p(\x)-\log q(\x)
=\tfrac{1}{2}\log\frac{\mathrm{det}(\mathbf{\Sigma}_Q)}{\mathrm{det}(\mathbf{\Sigma}_P)}
-\tfrac{1}{2}(\x-\boldsymbol{\mu}_P)^\top \mathbf{\Sigma}_P^{-1}(\x-\boldsymbol{\mu}_P)
+\tfrac{1}{2}(\x-\boldsymbol{\mu}_Q)^\top \mathbf{\Sigma}_Q^{-1}(\x-\boldsymbol{\mu}_Q),
\]
because the $-\tfrac{d}{2}\log(2\pi)$ terms cancel. Taking $\mathbb{E}_P[\cdot]$ on both sides yields
\[
\mathrm{KL}(P\|Q)
=\tfrac{1}{2}\log\frac{\mathrm{det}(\mathbf{\Sigma}_Q)}{\mathrm{det}(\mathbf{\Sigma}_P)}
+\tfrac{1}{2}\,\mathbb{E}_{P}\!\left[
-(\x-\boldsymbol{\mu}_P)^\top \mathbf{\Sigma}_P^{-1}(\x-\boldsymbol{\mu}_P)
+(\x-\boldsymbol{\mu}_Q)^\top \mathbf{\Sigma}_Q^{-1}(\x-\boldsymbol{\mu}_Q)
\right].
\]

We now evaluate the two expectations using the quadratic-form identities from the previous subsection.
First,
\[
\mathbb{E}_{P}\!\left[(\x-\boldsymbol{\mu}_P)^\top \mathbf{\Sigma}_P^{-1}(\x-\boldsymbol{\mu}_P)\right]
=\mathrm{tr}(\mathbf{\Sigma}_P^{-1}\mathbf{\Sigma}_P)=d.
\]
For the second, write $\x-\boldsymbol{\mu}_Q=(\x-\boldsymbol{\mu}_P)+(\boldsymbol{\mu}_P-\boldsymbol{\mu}_Q)$ and expand:
\[
\begin{aligned}
&\mathbb{E}_{P}\!\left[(\x-\boldsymbol{\mu}_Q)^\top \mathbf{\Sigma}_Q^{-1}(\x-\boldsymbol{\mu}_Q)\right]\\
&\qquad=\mathbb{E}_{P}\!\left[(\x-\boldsymbol{\mu}_P)^\top \mathbf{\Sigma}_Q^{-1}(\x-\boldsymbol{\mu}_P)\right]
+2\,\mathbb{E}_{P}\!\left[(\x-\boldsymbol{\mu}_P)^\top \mathbf{\Sigma}_Q^{-1}(\boldsymbol{\mu}_P-\boldsymbol{\mu}_Q)\right]
+(\boldsymbol{\mu}_P-\boldsymbol{\mu}_Q)^\top \mathbf{\Sigma}_Q^{-1}(\boldsymbol{\mu}_P-\boldsymbol{\mu}_Q).
\end{aligned}
\]
The middle term vanishes because $\mathbb{E}_{P}[\x-\boldsymbol{\mu}_P]=\mathbf{0}$. For the first term, use the trace trick:
\[
\mathbb{E}_{P}\!\left[(\x-\boldsymbol{\mu}_P)^\top \mathbf{\Sigma}_Q^{-1}(\x-\boldsymbol{\mu}_P)\right]
=\mathrm{tr}\!\big(\mathbf{\Sigma}_Q^{-1}\,\mathbb{E}_{P}[(\x-\boldsymbol{\mu}_P)(\x-\boldsymbol{\mu}_P)^\top]\big)
=\mathrm{tr}(\mathbf{\Sigma}_Q^{-1}\mathbf{\Sigma}_P).
\]
Combining these pieces,
\[
\mathbb{E}_{P}\!\left[(\x-\boldsymbol{\mu}_Q)^\top \mathbf{\Sigma}_Q^{-1}(\x-\boldsymbol{\mu}_Q)\right]
=\mathrm{tr}(\mathbf{\Sigma}_Q^{-1}\mathbf{\Sigma}_P)
+(\boldsymbol{\mu}_P-\boldsymbol{\mu}_Q)^\top \mathbf{\Sigma}_Q^{-1}(\boldsymbol{\mu}_P-\boldsymbol{\mu}_Q).
\]

Substitute both expectations back into the expression for $\mathrm{KL}(P\|Q)$ and simplify:
\[
\mathrm{KL}(P\|Q)
=\tfrac{1}{2}\log\frac{\mathrm{det}(\mathbf{\Sigma}_Q)}{\mathrm{det}(\mathbf{\Sigma}_P)}
+\tfrac{1}{2}\Big(-d+\mathrm{tr}(\mathbf{\Sigma}_Q^{-1}\mathbf{\Sigma}_P)
+(\boldsymbol{\mu}_P-\boldsymbol{\mu}_Q)^\top \mathbf{\Sigma}_Q^{-1}(\boldsymbol{\mu}_P-\boldsymbol{\mu}_Q)\Big),
\]
which matches \eqref{eq:kl-general} (note the mean term is symmetric in $P,Q$). This completes the proof.
\end{proof}

\paragraph{Special cases.}
If both covariances are diagonal, \eqref{eq:kl-general} reduces componentwise:
\[
\mathrm{KL}(P\|Q)
=\tfrac{1}{2}\sum_{i=1}^d\!\left(
\frac{\sigma_{P,i}^2}{\sigma_{Q,i}^2}
+\frac{(\mu_{P,i}-\mu_{Q,i})^2}{\sigma_{Q,i}^2}
-1
+\log\frac{\sigma_{Q,i}^2}{\sigma_{P,i}^2}
\right).
\]
If both are isotropic with the same variance, $\mathbf{\Sigma}_P=\mathbf{\Sigma}_Q=\sigma^2\mathbf{I}_d$, then
$\mathrm{KL}(P\|Q)=\tfrac{1}{2\sigma^2}\|\boldsymbol{\mu}_P-\boldsymbol{\mu}_Q\|_2^2$.
If they are isotropic with different variances,
\[
\mathrm{KL}(P\|Q)
=\tfrac{1}{2}\left(
d\Big(\frac{\sigma_P^2}{\sigma_Q^2}-1-\log\frac{\sigma_P^2}{\sigma_Q^2}\Big)
+\frac{1}{\sigma_Q^2}\,\|\boldsymbol{\mu}_P-\boldsymbol{\mu}_Q\|_2^2
\right).
\]

\section{Diffusion Models}
\label{sec:diffusion}

\noindent
A diffusion model has two coupled stochastic processes over data vectors $\x \in \mathbb{R}^d$:
\begin{itemize}
    \item a \emph{forward} (noising) process $q$ that gradually destroys structure by adding Gaussian noise in $T$ small steps; and
    \item a \emph{reverse} (denoising) process $p_\theta$ that learns to invert those steps to recover data from noise.
\end{itemize}

\subsection{The forward process}

\noindent
The joint probability of the whole path of this forward process follows the Markov property:
\[
\underbrace{{\qwholecond}}_{\displaystyle\textrm{Denote with } {q(\x_{1:T} \mid \x_0)}} = q(\x_T \mid \x_{T-1}) ~\cdots ~q(\x_3 \mid \x_2)~ q(\x_2 \mid \x_1)~ q(\x_1 \mid \x_0)\;.
\]

\noindent The Gaussians $\q$ define the forward diffusion \cite{sohl-dickstein-2015}:

\begin{equation}
    \label{eq:fwd}
    q(\x_t \mid \x_{t-1}) \;=\; \mathcal{N}\!\big(\x_t;\,\sqrt{1-\beta_t}\,\x_{t-1},\,\beta_t\,\mathbf{I}\big)\,.
\end{equation}

where $\sqrt{1-\beta_t}\,\x_{t-1}$ is the mean and $\beta_t \mathbf{I}$ is the covariance. Here $\mathbf{I}$ is the identity matrix and $\beta_t \in (0, 1)$ is the noise-level parameter in a pre-determined noise schedule
\[
\beta_1 < \beta_ 2 < \cdots < \beta_T\;.
\]

The forward process is designed so that the terminal state is (very close to) standard Gaussian. Define:

\begin{equation}
    \label{eq:notations}
\alpha_t := 1-\beta_t\;, \quad \quad \bar{\alpha}_t := \prod_{i=1}^t \alpha_i\;,
\end{equation}

then, with the cumulative product $\bar{\alpha}_T=\prod_{i=1}^T \alpha_i$ shrinking towards $0$
as $T$ grows, the marginal $\displaystyle q(\x_T\mid \x_0)=\mathcal{N}\!\big(\sqrt{\bar{\alpha}_T}\,\x_0,\,(1-\bar{\alpha}_T)\mathbf{I}\big)$
approaches $\mathcal N(\mathbf{0},\mathbf{I})$, so we set $p(\x_T)=\mathcal N(\mathbf{0},\mathbf{I})$. With this, we can re-write \eqref{eq:fwd} as
\begin{equation}
    \label{eq:fwdproc}
    \q = \mathcal{N}\!\bigg(\x_t;~\sqrt{\alpha_t}\,\x_{t-1},~(1-\alpha_t)\,\mathbf{I}\bigg).
\end{equation}

We will show that to generate $\x_t$ we only need $\x_0$ (plus fresh Gaussian noise). Using \eqref{eq:fwdproc} and the re-parameterisation trick of Corollary \ref{cor:full_cov_sampling_easy}, we write
\begin{align}
    \label{eq:x2}
    \x_t &= \sqrt{\alpha_t}\,\x_{t-1} + \sqrt{1-\alpha_t}\, \boldsymbol{\epsilon}\;, \nonumber\\
    \x_{t-1} &= \sqrt{\alpha_{t-1}}\,\x_{t-2} + \sqrt{1-\alpha_{t-1}}\,\boldsymbol{\epsilon}\;, \nonumber \\
    \x_{t-2} &= \sqrt{\alpha_{t-2}}\,\x_{t-3} + \sqrt{1-\alpha_{t-2}}\,\boldsymbol{\epsilon}\;,\\
    \vdots \nonumber\\
    \x_1 &= \sqrt{\alpha_1}\, \x_0 + \sqrt{1-\alpha_1}\, \boldsymbol{\epsilon}\;, \nonumber
\end{align}
where each $\boldsymbol{\epsilon}\sim \mathcal{N}(\mathbf{0},\mathbf{I})$ is an independent draw. Substitute $\x_{t-1}$ into the first line to get

\begin{align}
\label{eq:sampling}
    \x_t &= \sqrt{\alpha_t} \Big(\sqrt{\alpha_{t-1}}\,\x_{t-2} + \sqrt{1-\alpha_{t-1}}\, \boldsymbol{\epsilon}\Big) + \sqrt{1-\alpha_t}\, \boldsymbol{\epsilon} \nonumber \\
    &= \sqrt{\alpha_t \alpha_{t-1}}\, \x_{t-2} + \underbrace{\sqrt{\alpha_t (1-\alpha_{t-1})}\,\boldsymbol{\epsilon}}_{\mathcal{N}(\mathbf{0},\,\alpha_t(1-\alpha_{t-1})\mathbf{I})} + \underbrace{\sqrt{1-\alpha_t}\,\boldsymbol{\epsilon}}_{\mathcal{N}(\mathbf{0},\,(1-\alpha_t)\mathbf{I})}\;.
\end{align}
The last two terms are independent Gaussians, so their sum is Gaussian with variance equal to the sum of variances:

\[
\sqrt{\alpha_t (1-\alpha_{t-1})}\,\boldsymbol{\epsilon} + \sqrt{1-\alpha_t}\,\boldsymbol{\epsilon} \sim \mathcal{N}\!\bigg(\mathbf{0},\, (1-\alpha_t\alpha_{t-1})\mathbf{I}\bigg)\;,
\]

and thus
\begin{equation}
\label{eq:sampleX2}
\x_t = \sqrt{\alpha_t \alpha_{t-1}}\, \x_{t-2} + \sqrt{1-\alpha_t \alpha_{t-1}}\, \boldsymbol{\epsilon}\;.
\end{equation}
Repeating once more by replacing $\x_{t-2}$,
\begin{align}
    \label{eq:x3}
    \x_t &= \sqrt{\alpha_t \alpha_{t-1}} \Big(\sqrt{\alpha_{t-2}}\,\x_{t-3} + \sqrt{1-\alpha_{t-2}}\,\boldsymbol{\epsilon}\Big) + \sqrt{1-\alpha_t \alpha_{t-1}}\, \boldsymbol{\epsilon} \nonumber \\
    &= \sqrt{\alpha_t \alpha_{t-1} \alpha_{t-2}}\, \x_{t-3} + \underbrace{\sqrt{\alpha_t \alpha_{t-1} (1-\alpha_{t-2})}\, \boldsymbol{\epsilon}}_{\mathcal{N}(\mathbf{0},\,\alpha_t \alpha_{t-1} (1-\alpha_{t-2})\mathbf{I})} + \underbrace{\sqrt{1-\alpha_t \alpha_{t-1}}\, \boldsymbol{\epsilon}}_{\mathcal{N}(\mathbf{0},(\,1-\alpha_t \alpha_{t-1})\mathbf{I})}\;,
\end{align}
whose sum is again Gaussian with variance $1-\alpha_t \alpha_{t-1} \alpha_{t-2}$. Hence
\[
    \x_t = \sqrt{\alpha_t \alpha_{t-1} \alpha_{t-2}}\, \x_{t-3} + \sqrt{1-\alpha_t \alpha_{t-1} \alpha_{t-2}}\, \boldsymbol{\epsilon}\;.
\]
By induction, continuing until $\x_t$ is expressed in terms of $\x_0$ yields
\[
\x_t = \sqrt{\prod_{i=1}^t \alpha_i}\, \x_0 + \sqrt{1- \prod_{i=1}^t \alpha_i}\, \boldsymbol{\epsilon}\;.
\]
Using the notation in \eqref{eq:notations}, the closed-form sampling equation at any time $t$ is
\begin{equation}
    \label{eq:samplingeq}
    \x_t = \sqrt{\bar{\alpha}_t}\, \x_0 + \sqrt{1-\bar{\alpha}_t}\, \boldsymbol{\epsilon}\;.
\end{equation}
Therefore, since $\x_0$ here is treated as given (non-random under the conditional), the distribution of $\x_t$ has mean $\sqrt{\bar{\alpha}_t}\,\x_0$ and covariance $(1-\bar{\alpha}_t)\,\mathbf{I}$:
\begin{equation}
    \label{eq:samplingdist}
q(\x_t \mid \x_0) = \mathcal{N}\!\bigg(\x_t\;;\,\sqrt{\bar{\alpha}_t}\,\x_0\;,\, (1-\bar{\alpha}_t)\,\mathbf{I}\bigg)\;.
\end{equation}
\medskip

\subsection{The reverse process and the true posterior (DDPM posterior)}

\noindent
The joint probability of the whole reverse Markov process is:\bigskip

\[
\begin{tikzpicture}[overlay]
\node at (8, 1) {\footnotesize{\color{blue}{We start by drawing $\displaystyle \mathbf{x}_T$ from $\displaystyle \pt$}}};
\draw[blue, ->] (9,.8) -- (9,0.3);
\end{tikzpicture}
\underbrace{\pt(\mathbf{x}_T, \mathbf{x}_{T-1}, \dots, \mathbf{x}_1, \mathbf{x}_0)}_{\displaystyle\textrm{Denote with } \pt(\mathbf{x}_{0:T})}
= \pt(\mathbf{x}_0\mid\mathbf{x}_1)\; \cdots\; \pt(\mathbf{x}_{T-1}\mid\mathbf{x}_T)\; \pt(\mathbf{x}_T)\;.
\]

\noindent The true posterior (DDPM posterior) can be written as a Gaussian:
\begin{equation}
\label{eq:grndt}
q(\mathbf{x}_{t-1} \mid \mathbf{x}_t, \mathbf{x}_0)
= \mathcal{N}\!\big(\mathbf{x}_{t-1}\,;\, \hat{\boldsymbol{\mu}}(\mathbf{x}_t, \mathbf{x}_0),\, \hat{\mathbf{\Sigma}}(\mathbf{x}_t, \mathbf{x}_0)\big)\;,
\end{equation}
where $\hat{\boldsymbol{\mu}}(\mathbf{x}_t, \mathbf{x}_0) \in \mathbb{R}^d$ and $\hat{\mathbf{\Sigma}}(\mathbf{x}_t, \mathbf{x}_0) \in \mathbb{R}^{d \times d}$ are the mean and covariance. Because the forward noise is isotropic, the \emph{true} posterior covariance is also isotropic:
\[
\hat{\mathbf{\Sigma}}(\mathbf{x}_t, \mathbf{x}_0)=\hat{\sigma}_t^2\,\mathbf{I}_d
\qquad\text{with}\qquad
\hat{\sigma}_t^2=\tilde{\beta}_t
:= \frac{1-\bar{\alpha}_{t-1}}{1-\bar{\alpha}_t}\,\beta_t\;.
\]
We derive both $\hat{\boldsymbol{\mu}}$ and $\hat{\sigma}_t^2$ below.

\medskip
By Bayes’ rule,
\[
q(\mathbf{x}_{t-1} \mid \mathbf{x}_t, \mathbf{x}_0)
= \frac{q(\mathbf{x}_t \mid \mathbf{x}_{t-1}, \mathbf{x}_0)\, q(\mathbf{x}_{t-1} \mid \mathbf{x}_0)}{q(\mathbf{x}_t \mid \mathbf{x}_0)}\;.
\]
Since the forward process is Markovian,
\[
q(\mathbf{x}_t \mid \mathbf{x}_{t-1}, \mathbf{x}_0) = q(\mathbf{x}_t \mid \mathbf{x}_{t-1})
= \mathcal{N}\!\big(\sqrt{\alpha_t}\,\mathbf{x}_{t-1},\,\beta_t \mathbf{I}_d\big)
\propto \exp\!\Big(-\tfrac{1}{2}\, \tfrac{\|\mathbf{x}_t - \sqrt{\alpha_t}\,\mathbf{x}_{t-1}\|^2}{\beta_t}\Big)\;,
\]
and from \eqref{eq:samplingdist},
\begin{align*}
q(\mathbf{x}_{t-1} \mid \mathbf{x}_0)
&= \mathcal{N}\!\big(\sqrt{\bar{\alpha}_{t-1}}\,\mathbf{x}_0,\ (1-\bar{\alpha}_{t-1})\,\mathbf{I}_d\big)
\propto \exp\!\Big(-\tfrac{1}{2}\, \tfrac{\|\mathbf{x}_{t-1} - \sqrt{\bar{\alpha}_{t-1}}\,\mathbf{x}_{0}\|^2}{1-\bar{\alpha}_{t-1}}\Big)\;,\\[0.6em]
q(\mathbf{x}_{t} \mid \mathbf{x}_0)
&= \mathcal{N}\!\big(\sqrt{\bar{\alpha}_t}\,\mathbf{x}_0,\ (1-\bar{\alpha}_t)\,\mathbf{I}_d\big)
\propto \exp\!\Big(-\tfrac{1}{2}\, \tfrac{\|\mathbf{x}_t - \sqrt{\bar{\alpha}_t}\,\mathbf{x}_{0}\|^2}{1-\bar{\alpha}_{t}}\Big)\;.
\end{align*}
Therefore,
\begin{align}
\label{eq:gndt2}
q(\mathbf{x}_{t-1} \mid \mathbf{x}_t, \mathbf{x}_0)
&\propto \exp\!\Big( -\tfrac{1}{2} \Big[\tfrac{\|\mathbf{x}_t - \sqrt{\alpha_t}\,\mathbf{x}_{t-1}\|^2}{\beta_t}
+ \tfrac{\|\mathbf{x}_{t-1} - \sqrt{\bar{\alpha}_{t-1}}\,\mathbf{x}_{0}\|^2}{1-\bar{\alpha}_{t-1}}
- \tfrac{\|\mathbf{x}_t - \sqrt{\bar{\alpha}_t}\,\mathbf{x}_{0}\|^2}{1-\bar{\alpha}_{t}} \Big] \Big)\;.
\end{align}
Expanding the quadratic terms,
\begin{align*}
\frac{\|\mathbf{x}_t - \sqrt{\alpha_t}\,\mathbf{x}_{t-1}\|^2}{\beta_{t}}
&= \frac{\mathbf{x}_t^\top \mathbf{x}_t -2\sqrt{\alpha_t}\, \mathbf{x}_t^\top \mathbf{x}_{t-1} + \alpha_t\, \mathbf{x}_{t-1}^\top \mathbf{x}_{t-1}}{\beta_t}\;,\\
\frac{\|\mathbf{x}_{t-1} - \sqrt{\bar{\alpha}_{t-1}}\,\mathbf{x}_{0}\|^2}{1-\bar{\alpha}_{t-1}}
&= \frac{\mathbf{x}_{t-1}^\top \mathbf{x}_{t-1} -2\sqrt{\bar{\alpha}_{t-1}}\, \mathbf{x}_{t-1}^\top \mathbf{x}_{0} + \bar{\alpha}_{t-1}\, \mathbf{x}_{0}^\top \mathbf{x}_{0}}{1-\bar{\alpha}_{t-1}}\;,\\
\frac{\|\mathbf{x}_t - \sqrt{\bar{\alpha}_t}\,\mathbf{x}_{0}\|^2}{1-\bar{\alpha}_{t}}
&= \frac{\mathbf{x}_{t}^\top \mathbf{x}_{t} -2\sqrt{\bar{\alpha}_t}\, \mathbf{x}_{t}^\top \mathbf{x}_{0} + \bar{\alpha}_t\, \mathbf{x}_{0}^\top \mathbf{x}_{0}}{1-\bar{\alpha}_{t}}\;.
\end{align*}
Thus,
\begin{align*}
q(\mathbf{x}_{t-1} \mid \mathbf{x}_t, \mathbf{x}_0)
\propto \exp\!\Big( -\tfrac{1}{2} \Big[&
\frac{\mathbf{x}_t^\top \mathbf{x}_t -2\sqrt{\alpha_t}\, \mathbf{x}_t^\top \mathbf{x}_{t-1} + \alpha_t\, \mathbf{x}_{t-1}^\top \mathbf{x}_{t-1}}{\beta_t}
+ \frac{\mathbf{x}_{t-1}^\top \mathbf{x}_{t-1} -2\sqrt{\bar{\alpha}_{t-1}}\, \mathbf{x}_{t-1}^\top \mathbf{x}_{0} + \bar{\alpha}_{t-1}\, \mathbf{x}_{0}^\top \mathbf{x}_{0}}{1-\bar{\alpha}_{t-1}}\\
&- \frac{\mathbf{x}_{t}^\top \mathbf{x}_{t} -2\sqrt{\bar{\alpha}_t}\, \mathbf{x}_{t}^\top \mathbf{x}_{0} + \bar{\alpha}_t\, \mathbf{x}_{0}^\top \mathbf{x}_{0}}{1-\bar{\alpha}_{t}}\Big]\Big)\;,
\end{align*}
which collects into a quadratic form in $\mathbf{x}_{t-1}$:
\begin{align*}
q(\mathbf{x}_{t-1} \mid \mathbf{x}_t, \mathbf{x}_0)
\propto \exp\!\Big( -\tfrac{1}{2}\Big[&
\Big(\tfrac{\alpha_t}{\beta_t} + \tfrac{1}{1-\bar{\alpha}_{t-1}}\Big)\mathbf{x}_{t-1}^\top \mathbf{x}_{t-1}
-2\Big(\tfrac{\sqrt{\alpha_t}}{\beta_t}\mathbf{x}_t + \tfrac{\sqrt{\bar{\alpha}_{t-1}}}{1-\bar{\alpha}_{t-1}}\mathbf{x}_0\Big)^\top \mathbf{x}_{t-1}
+ C(\mathbf{x}_t, \mathbf{x}_0)\Big] \Big)\;,
\end{align*}
where $C(\mathbf{x}_t, \mathbf{x}_0)$ contains no $\mathbf{x}_{t-1}$ terms (hence is irrelevant for the Gaussian identification). Define
\[
A := \frac{\alpha_t}{\beta_t} +  \frac{1}{1-\bar{\alpha}_{t-1}}\;, \qquad 
\mathbf{b} := \frac{\sqrt{\alpha_t}}{\beta_t}\mathbf{x}_t + \frac{\sqrt{\bar{\alpha}_{t-1}}}{1-\bar{\alpha}_{t-1}}\mathbf{x}_0\;.
\]
Completing the square gives
\[
q(\mathbf{x}_{t-1} \mid \mathbf{x}_t, \mathbf{x}_0)
\propto \exp \!\Big(-\tfrac{1}{2} A \big\|\mathbf{x}_{t-1} - \tfrac{\mathbf{b}}{A}\big\|_2^2 + \tfrac{\mathbf{b}^\top \mathbf{b}}{2A}\Big)\;,
\]
so the \emph{posterior mean} is
\begin{align}
\label{eq:meangnd}
\hat{\boldsymbol{\mu}}(\mathbf{x}_t, \mathbf{x}_0)
= \frac{\mathbf{b}}{A}
= \frac{\frac{\sqrt{\alpha_t}}{\beta_t}\,\mathbf{x}_t
      + \frac{\sqrt{\bar{\alpha}_{t-1}}}{1-\bar{\alpha}_{t-1}}\,\mathbf{x}_0}
     {\frac{\alpha_t}{\beta_t} + \frac{1}{1-\bar{\alpha}_{t-1}}}
= \frac{\sqrt{\alpha_t}(1-\bar{\alpha}_{t-1})}{1-\bar{\alpha}_t}\,\mathbf{x}_t
\;+\;
\frac{\sqrt{\bar{\alpha}_{t-1}}\,\beta_t}{1-\bar{\alpha}_t}\,\mathbf{x}_0 \;,
\end{align}
and the \emph{posterior covariance} is the inverse of the scalar precision $A$ times $\mathbf{I}_d$:
\[
\hat{\mathbf{\Sigma}}(\mathbf{x}_t,\mathbf{x}_0) \;=\; A^{-1}\,\mathbf{I}_d
\qquad\text{with}\qquad
A^{-1}
= \frac{1}{\frac{\alpha_t}{\beta_t} + \frac{1}{1-\bar{\alpha}_{t-1}}}
= \frac{\beta_t(1-\bar{\alpha}_{t-1})}{\alpha_t(1-\bar{\alpha}_{t-1})+\beta_t}
= \frac{\beta_t(1-\bar{\alpha}_{t-1})}{1-\bar{\alpha}_t}
= \tilde{\beta}_t\;.
\]
Thus
\[
q(\mathbf{x}_{t-1} \mid \mathbf{x}_t, \mathbf{x}_0)
= \mathcal{N}\!\big(\hat{\boldsymbol{\mu}}(\mathbf{x}_t,\mathbf{x}_0),\ \tilde{\beta}_t\,\mathbf{I}_d\big)\;.
\]

\medskip
From \eqref{eq:samplingeq} we can write
\[
\mathbf{x}_0 = \frac{1}{\sqrt{\bar{\alpha}_t}} \Big(\mathbf{x}_t - \sqrt{1-\bar{\alpha}_t}\, \boldsymbol{\epsilon}\Big)\;,
\]
and substitute into \eqref{eq:meangnd} to obtain the convenient $\boldsymbol{\epsilon}$–form
\begin{equation}
\label{eq:gndtmean}
\hat{\boldsymbol{\mu}}(\mathbf{x}_t, \mathbf{x}_0)
= \frac{1}{\sqrt{\alpha_t}} \Big(\mathbf{x}_t - \frac{\beta_t}{\sqrt{1-\bar{\alpha}_t}}\, \boldsymbol{\epsilon}\Big)\;.
\end{equation}

\subsection{The loss function}

\noindent
We are interested in $\pt(\x_0)$, the probability assigned by the learned model to a generated $\x_0$. Start with
\[
\pt(\x_{1:T} \mid \x_0)\,\pt(\x_0) = \pt(\x_{0:T})\;.
\]
Although it appears anti-causal, the identity is algebraically valid. Our aim is to make $q(\x_{1:T} \mid \x_0)$ and $\pt(\x_{1:T} \mid \x_0)$ close, via the non-negative KL divergence
\[
0 \leq \mathrm{KL}\!\Big(q(\x_{1:T} \mid \x_0)\;||\;\pt(\x_{1:T} \mid \x_0)\Big) = \mathbb{E}_{q(\x_{1:T}\mid\x_0)} \log \frac{q(\x_{1:T}\mid\x_0)}{\pt(\x_{1:T}\mid\x_0)}\;.
\]
Simplify the expectation:
\begin{align*}
    \mathrm{KL}\!\Big(q(\x_{1:T} \mid \x_0)\;||\;\pt(\x_{1:T} \mid \x_0)\Big) &= \mathbb{E}_{q(\x_{1:T}\mid\x_0)} \log \frac{q(\x_{1:T}\mid\x_0)}{\color{blue}{\pt(\x_{1:T}\mid\x_0)}}\\[0.8em]
&= \mathbb{E}_{q(\x_{1:T}\mid\x_0)} \log \frac{q(\x_{1:T}\mid\x_0)\, \color{blue}{\pt(\x_0)}}{\color{blue}{\pt(\x_{0:T})}}\\[0.8em]
&= \mathbb{E}_{q(\x_{1:T}\mid\x_0)} \log \frac{q(\x_{1:T}\mid\x_0)}{\pt(\x_{0:T})} + {\mathbb{E}_{q(\x_{1:T}\mid\x_0)} \log \pt(\x_0)}\;.
\end{align*}
Because $\log \pt(\x_0)$ does not depend on $\x_{1:T}$,
\[
    {\mathbb{E}_{q(\x_{1:T}\mid\x_0)} \log \pt(\x_0)} = \log \pt(\x_0)\;.
\]
Hence
\begin{align*}
    \mathrm{KL}\!\Big(q(\x_{1:T} \mid \x_0)\;||\;\pt(\x_{1:T} \mid \x_0)\Big) &= \mathbb{E}_{q(\x_{1:T}\mid\x_0)} \log \frac{q(\x_{1:T}\mid\x_0)}{\pt(\x_{0:T})} + \log \pt(\x_0)\ge 0\\[0.6em]
    & \Longrightarrow \quad \mathbb{E}_{q(\x_{1:T}\mid\x_0)} \log \frac{q(\x_{1:T}\mid\x_0)}{\pt(\x_{0:T})} \ge -\log \pt(\x_0)\;.
\end{align*}
Taking expectation over $q(\x_0)$,
\[
    \mathbb{E}_{q(\x_{0:T})} \log \frac{q(\x_{1:T}\mid\x_0)}{\pt(\x_{0:T})} \ge {-\mathbb{E}_{q(\x_0)}\log \pt(\x_0)}\;.
\]
The right-hand side relates to $\mathrm{KL}\!\big(q(\x_0)\,||\,\pt(\x_0)\big) = \mathrm{\mathcolor{blue}{\mathbf{const}}} - \mathbb{E}_{q(\x_{0})} \log \pt(\x_0)$, which we would like to minimise, but $\log \pt(\x_0)$ is intractable. Instead, we minimise the left-hand side, the \emph{Variational Bound (VB)} (closely related to the Evidence Lower Bound in VAEs). Expanding,
\begin{align*}
    \mathbb{E}_{q(\x_{0:T})} \log \frac{q(\x_{1:T}\mid\x_0)}{\pt(\x_{0:T})} &= \mathbb{E}_{q(\x_{0:T})} \log \frac{q(\x_T \mid \x_{T-1})~ \cdots ~q(\x_3 \mid \x_{2})~ q(\x_2 \mid \x_{1})~ \mathcolor{blue}{q(\x_1 \mid \x_{0})}}{\mathcolor{blue}{\pt(\x_0\mid\x_1)}~ \cdots~ \pt(\x_{T-1}\mid\x_T)~ \mathcolor{Purple}{\pt(\x_T)}}\\[0.6em]
    &= \mathbb{E}_{q(\x_{0:T})} \log \frac{q(\x_1 \mid \x_0)}{\pt(\x_0 \mid \x_1)} - \mathbb{E}_{q(\x_{0:T})} \log \pt(\x_T) + \mathbb{E}_{q(\x_{0:T})} \sum_{t=2}^T \log \frac{q(\x_t \mid \x_{t-1})}{\pt(\x_{t-1} \mid \x_t)}\\
    & = \mathbb{E}_{q(\x_{0:T})} \log \frac{q(\x_1 \mid \x_0)}{\pt(\x_0 \mid \x_1)} - \mathbb{E}_{q(\x_{0:T})} \log \pt(\x_T) + {\mathbb{E}_{q(\x_{0:T})} \sum_{t=2}^T \log \frac{q(\x_t \mid \x_{t-1}, \mathcolor{red}{\x_0})}{\pt(\x_{t-1} \mid \x_t)}}\;.
\end{align*}
Applying Bayes’ theorem,
\begin{align*}
    \mathrm{VB} &= \mathbb{E}_{q(\x_{0:T})} \log \frac{q(\x_1 \mid \x_0)}{\pt(\x_0 \mid \x_1)} - \mathbb{E}_{q(\x_{0:T})} \log \pt(\x_T) + {\mathbb{E}_{q(\x_{0:T})} \sum_{t=2}^T \log \frac{\mathcolor{red}{q(\x_{t-1} \mid \x_{t}, \x_0)\, q(\x_t \mid \x_0)}}{\pt(\x_{t-1} \mid \x_t)\, \mathcolor{red}{q(\x_{t-1} \mid \x_0)}}}\\[0.4em]
    &=\mathbb{E}_{q(\x_{0:T})} \log \frac{q(\x_1 \mid \x_0)}{\pt(\x_0 \mid \x_1)} - \mathbb{E}_{q(\x_{0:T})} \log \pt(\x_T) + {\mathbb{E}_{q(\x_{0:T})} \sum_{t=2}^T \log \frac{q(\x_{t-1} \mid \x_{t}, \x_0)}{\pt(\x_{t-1} \mid \x_t)}}\\
    &\qquad\qquad + {\mathbb{E}_{q(\x_{0:T})} \sum_{t=2}^T \log \frac{q(\x_t \mid \x_0)}{q(\x_{t-1} \mid \x_0)}}\;.
\end{align*}
The telescoping sum simplifies to
\[
    \mathbb{E}_{q(\x_{0:T})} \sum_{t=2}^T \log \frac{q(\x_t \mid \x_0)}{q(\x_{t-1} \mid \x_0)} = \mathbb{E}_{q(\x_{0:T})} \log \frac{q(\x_T \mid \x_0)}{q(\x_1 \mid \x_0)}\;.
\]
Thus
\begin{align*}
    \begin{tikzpicture}[overlay]
        \node at (3, -2.1) {\textcolor{black}{\circled{\textbf{I}}}};
        \node at (7, -2.3) {\textcolor{black}{\circled{\textbf{II}}}};
        \node at (12, -2) {\textcolor{black}{\circled{\textbf{III}}}};
    \end{tikzpicture}
    \mathrm{VB} &= \mathbb{E}_{q(\x_{0:T})} \log \frac{\Ccancel{q(\x_1 \mid \x_0)}}{\pt(\x_0 \mid \x_1)} ~-~ \mathbb{E}_{q(\x_{0:T})} \log \pt(\x_T) ~+~ \mathcolor{blue}{\mathbb{E}_{q(\x_{0:T})} \log \frac{q(\x_T \mid \x_0)}{\Ccancel{q(\x_1 \mid \x_0)}}}\\ &\qquad +~ \mathbb{E}_{q(\x_{0:T})} \sum_{t=2}^T \log \frac{q(\x_{t-1} \mid \x_{t}, \x_0)}{\pt(\x_{t-1} \mid \x_t)}\\[1.8em]
    &= \mathbb{E}_{q(\x_{0:T})} \log \frac{q(\x_T \mid \x_0)}{\pt(\x_T)} ~-~ \mathbb{E}_{q(\x_{0:T})} \log \pt(\x_0 \mid \x_1) ~+~ \mathbb{E}_{q(\x_{0:T})} \sum_{t=2}^T \log \frac{q(\x_{t-1} \mid \x_{t}, \x_0)}{\pt(\x_{t-1} \mid \x_t)}\;.
\end{align*}

\begin{itemize}
    \item[\textcolor{black}{\circled{\textbf{I}}}] $\displaystyle \mathbb{E}_{q(\x_{0:T})} \log \frac{q(\x_T \mid \x_0)}{\pt(\x_T)} = 0$ when $\displaystyle q(\x_T \mid \x_0) = p_\theta(\x_T)$; otherwise it is a small constant in practice when $p_\theta(\x_T)=\mathcal{N}(\mathbf{0},\mathbf{I})$ and $T$ is large.\bigskip
    \item[\textcolor{black}{\circled{\textbf{II}}}] $\displaystyle \mathbb{E}_{q(\x_{0:T})} \log \pt(\x_0 \mid \x_1)$ is commonly neglected in practice (acts like a $t{=}1$ decoder term).
\bigskip
    \item[\textcolor{black}{\circled{\textbf{III}}}] $\displaystyle \mathbb{E}_{q(\x_{0:T})} \sum_{t=2}^T \log \frac{q(\x_{t-1} \mid \x_{t}, \x_0)}{\pt(\x_{t-1} \mid \x_t)}$ reduces to a sum of KL divergences between isotropic Gaussians (closed form).
\end{itemize}
\bigskip

\noindent
Hence the optimisation encourages $\pt(\x_{t-1} \mid \x_t)$ to match $q(\x_{t-1} \mid \x_t, \x_0)$ for $t=2,\dots,T$ \cite{ho-2020-ddpm}:
\begin{equation}
\mathrm{VB} = \mathbb{E}_{q(\x_{0:T})} \sum_{t=2}^T \log \frac{q(\x_{t-1} \mid \x_{t}, \x_0)}{\pt(\x_{t-1} \mid \x_t)} \;\cong\; \sum_{t=2}^T \mathrm{KL}\!\Big\{\mathcolor{purple}{q(\x_{t-1} \mid \x_t, \x_0)} ~\big\|~ \mathcolor{blue}{\pt(\x_{t-1} \mid \x_t)}\Big\}.
\label{eq:VB}
\end{equation}
These two Gaussians are
\begin{itemize}
    \item $\displaystyle \mathcolor{purple}{q(\x_{t-1} \mid \x_t, \x_0) = \mathcal{N} \!\bigg(\x_{t-1};~ \frac{1}{\sqrt{\alpha_t}}\Big(\x_t - \frac{\beta_t}{\sqrt{1-\bar{\alpha}_t}} \boldsymbol{\epsilon}_t\Big),~ {\frac{1-\bar{\alpha}_{t-1}}{1-\bar{\alpha}_t}\, \beta_t}\, \mathbf{I}\bigg)}$\;, \bigskip
    \item $\displaystyle \mathcolor{blue}{\pt(\x_{t-1} \mid \x_t) = \mathcal{N}\!\bigg(\x_{t-1};~ \mu_\theta(\x_t, t),~ {\sigma_t^2}\, \mathbf{I}\bigg)}$\;.
\end{itemize}
Set $\displaystyle \sigma_t^2 = \frac{1-\bar{\alpha}_{t-1}}{1-\bar{\alpha}_t}\,\beta_t \;:=\; \tilde{\beta}_t$.
 Then
\begin{align*}
    \mathrm{KL}\!\Big\{\mathcolor{purple}{q(\x_{t-1} \mid \x_t, \x_0)} ~\big\|~ \mathcolor{blue}{\pt(\x_{t-1} \mid \x_t)}\Big\} &= \frac{1}{2\sigma_t^2}\left\|\mathcolor{purple}{\frac{1}{\sqrt{1-\beta_t}}\Big(\x_t - \frac{\beta_t}{\sqrt{1-\alpha_t}} \boldsymbol{\epsilon}_t\Big)} - \mathcolor{blue}{\mu_\theta(\x_t, t)}\right\|_2^2\\[0.3em]
    & = \frac{1-\bar{\alpha}_t}{2\beta_t(1-\bar{\alpha}_{t-1})}\left\|\mathcolor{purple}{\frac{1}{\sqrt{1-\beta_t}}\Big(\x_t - \frac{\beta_t}{\sqrt{1-\alpha_t}} \boldsymbol{\epsilon}_t\Big)} - \mathcolor{blue}{\mu_\theta(\x_t, t)}\right\|_2^2\;.
\end{align*}
Using the parameterisation \cite{nichol-2021-improved}
\[
    \mathcolor{blue}{\mu_\theta(\x_t, t) = \frac{1}{\sqrt{\alpha_t}}\Big(\x_t - \frac{\beta_t}{\sqrt{1-\bar{\alpha}_t}} \hat{\boldsymbol{\epsilon}}_t\Big)}\;,
\]
we obtain the simple, weighted noise-matching objective
\begin{align*}
    \mathrm{KL}\!\Big\{\mathcolor{purple}{q(\x_{t-1} \mid \x_t, \x_0)} ~\big\|~ \mathcolor{blue}{\pt(\x_{t-1} \mid \x_t)}\Big\} ~=~ \frac{\beta_t}{2(1-\beta_t)(1-\bar{\alpha}_{t-1})}\, \big\|\mathcolor{purple}{\boldsymbol{\epsilon}_t} - \mathcolor{blue}{\hat{\boldsymbol{\epsilon}}_t}\big\|_2^2\;,
\end{align*}
and therefore
\begin{align*}
    \mathrm{VB} = \sum_{t=2}^T \frac{\beta_t}{2(1-\beta_t)(1-\bar{\alpha}_{t-1})}\, \big\|\mathcolor{purple}{\boldsymbol{\epsilon}_t} - \mathcolor{blue}{\hat{\boldsymbol{\epsilon}}_t}\big\|_2^2\;.
\end{align*}

\subsection{Computing \texorpdfstring{$p_{\theta}(\x_0)$}{p\_theta(x\_0)}}

\noindent
Write the joint probability of the complete reverse path using the Markov chain:
\begin{equation}
\label{eq:comprevpath}
p_\theta(\x_{0:T}) = p_\theta(\x_0\mid\x_1)\,p_\theta(\x_1\mid\x_2)\cdots p_\theta(\x_{T-1}\mid\x_T)\, p_\theta(\x_T)\;.
\end{equation}
Then $p_\theta(\x_0)$ is obtained by marginalising out $\x_{1:T}$:
\begin{equation}
    \label{eq:marginit}
    p_\theta(\x_0) = \int p_\theta(\x_{0:T})\,d\x_{1:T}\;.
\end{equation}
Direct computation is expensive, so we use importance sampling by multiplying by $\tfrac{q(\x_{1:T}\mid \x_0)}{q(\x_{1:T}\mid \x_0)}$:
\[
    p_\theta(\x_0) = \int p_\theta(\x_{0:T})\, \frac{q(\x_{1:T} \mid \x_0)}{q(\x_{1:T} \mid \x_0)} \,d\x_{1:T}
    = \int q(\x_{1:T} \mid \x_0)\, \frac{p_\theta(\x_{0:T})}{q(\x_{1:T} \mid \x_0)} \,d\x_{1:T}\;,
\]
which, after expanding terms, becomes
\[
    p_\theta(\x_0) = \int q(\x_{1:T} \mid \x_0)\, p_\theta(\x_T)\, \frac{\pt(\x_0\mid\x_1)\, \pt(\x_1\mid\x_2)\cdots \pt(\x_{T-1}\mid\x_T)}{q(\x_T \mid \x_{T-1}) \cdots q(\x_3 \mid \x_2)\, q(\x_2 \mid \x_1)\, q(\x_1 \mid \x_0)} \,d\x_{1:T}\;.
\]
By definition this is the expectation
\begin{equation}
    \label{eq:pxexpec}
    p_\theta(\x_0) = \mathbb{E}_{q(\x_{1:T} \mid \x_0)} \bigg[p_\theta(\x_T)\, \frac{\pt(\x_0\mid\x_1)\, \pt(\x_1\mid\x_2)\cdots \pt(\x_{T-1}\mid\x_T)}{q(\x_T \mid \x_{T-1}) \cdots q(\x_3 \mid \x_2)\, q(\x_2 \mid \x_1)\, q(\x_1 \mid \x_0)} \bigg]\;.
\end{equation}
Finally, recall the reverse-time conditionals used in practice:

\[
p_\theta(\x_{t-1} \mid \x_t) = \mathcal{N}\!\bigg(\x_{t-1}\;;\, \frac{1}{\sqrt{\alpha_t}}\Big(\x_t - \frac{\beta_t}{\sqrt{1-\bar{\alpha}_t}} \hat{\boldsymbol{\epsilon}}_t\Big),\, \frac{1-\bar{\alpha}_{t-1}}{1-\bar{\alpha}_t}\, \beta_t\, \mathbf{I}\bigg)\;.
\]

\section{Acceleration methods}

A critical drawback of DDPM models is that they require many iterations to produce a high quality sample \cite{karras-2022-edm}.

\subsection{Denoising Diffusion Implicit Models (DDIMs)}
\label{sec:DDIM}

We start with Denoising Diffusion Implicit Models (DDIMs) \cite{song-2021-ddim}. The starting point is the exact (DDPM) single–step posterior derived in the previous subsection:
\begin{equation}
q(\mathbf{x}_{t-1} \mid \mathbf{x}_t, \mathbf{x}_0)
= \mathcal{N}\!\Bigg(
\mathbf{x}_{t-1}\;;\;
\frac{1}{\sqrt{\alpha_t}}\!\left(\mathbf{x}_t - \frac{\beta_t}{\sqrt{1-\bar{\alpha}_t}}\,\boldsymbol{\epsilon}\right),
\;\sigma_t^2\,\mathbf{I}_d
\Bigg),
\qquad
\boldsymbol{\epsilon}
=
\frac{\mathbf{x}_t - \sqrt{\bar{\alpha}_t}\, \mathbf{x}_0}{\sqrt{1-\bar{\alpha}_t}},
\label{eq:ddpmreversed}
\end{equation}
where $\alpha_t=1-\beta_t$ and $\bar{\alpha}_t=\prod_{i=1}^t \alpha_i$. Setting $\sigma_t^2=\tilde{\beta}_t:=\frac{1-\bar{\alpha}_{t-1}}{1-\bar{\alpha}_t}\,\beta_t$ recovers the standard DDPM reverse conditional with its true posterior variance. Using the re-parameterisation of the forward marginal (Section~\ref{subsec:gaussian_reparam}),
\[
\mathbf{x}_t=\sqrt{\bar{\alpha}_t}\,\mathbf{x}_0+\sqrt{1-\bar{\alpha}_t}\,\boldsymbol{\epsilon}
\quad\Longrightarrow\quad
\boldsymbol{\epsilon}=\frac{\mathbf{x}_t-\sqrt{\bar{\alpha}_t}\,\mathbf{x}_0}{\sqrt{1-\bar{\alpha}_t}},
\]
substitute this $\boldsymbol{\epsilon}$ into the DDPM posterior mean in \eqref{eq:ddpmreversed}:
\begin{align*}
\hat{\boldsymbol{\mu}}_t^{\text{(DDPM)}}(\mathbf{x}_t,\mathbf{x}_0)
&= \frac{1}{\sqrt{\alpha_t}}\left(
\mathbf{x}_t - \frac{\beta_t}{\sqrt{1-\bar{\alpha}_t}}\cdot
\frac{\mathbf{x}_t-\sqrt{\bar{\alpha}_t}\,\mathbf{x}_0}{\sqrt{1-\bar{\alpha}_t}}
\right)\\
&= \frac{1}{\sqrt{\alpha_t}}\left(
\mathbf{x}_t - \frac{\beta_t}{1-\bar{\alpha}_t}\,\mathbf{x}_t
+ \frac{\beta_t\sqrt{\bar{\alpha}_t}}{1-\bar{\alpha}_t}\,\mathbf{x}_0
\right)\\
&= \underbrace{\frac{1}{\sqrt{\alpha_t}}\left(1-\frac{\beta_t}{1-\bar{\alpha}_t}\right)}_{\displaystyle c_t}\,\mathbf{x}_t
\;+\;
\underbrace{\frac{1}{\sqrt{\alpha_t}}\cdot \frac{\beta_t\sqrt{\bar{\alpha}_t}}{1-\bar{\alpha}_t}}_{\displaystyle d_t}\,\mathbf{x}_0.
\end{align*}
Now rewrite the coefficients $c_t$ and $d_t$ in terms of $(\alpha_t,\bar{\alpha}_{t-1},\bar{\alpha}_t)$. Since $\bar{\alpha}_t=\alpha_t\bar{\alpha}_{t-1}$,
\[
1-\bar{\alpha}_t
=
1-\alpha_t\bar{\alpha}_{t-1}
=
(1-\bar{\alpha}_{t-1})+\bar{\alpha}_{t-1}(1-\alpha_t)
=
(1-\bar{\alpha}_{t-1})+\bar{\alpha}_{t-1}\beta_t.
\]
Hence
\[
1-\frac{\beta_t}{1-\bar{\alpha}_t}
= \frac{1-\bar{\alpha}_t-\beta_t}{1-\bar{\alpha}_t}
= \frac{\alpha_t(1-\bar{\alpha}_{t-1})}{1-\bar{\alpha}_t},
\quad\text{so}\quad
c_t=\frac{\sqrt{\alpha_t}\,(1-\bar{\alpha}_{t-1})}{1-\bar{\alpha}_t},
\quad
d_t=\frac{\sqrt{\bar{\alpha}_{t-1}}\,\beta_t}{1-\bar{\alpha}_t}.
\]
This matches the “closed form” we derived earlier from completing the square:
\[
\hat{\boldsymbol{\mu}}_t^{\text{(DDPM)}}(\mathbf{x}_t,\mathbf{x}_0)
=
\frac{\sqrt{\alpha_t}(1-\bar{\alpha}_{t-1})}{1-\bar{\alpha}_t}\,\mathbf{x}_t
+
\frac{\sqrt{\bar{\alpha}_{t-1}}\beta_t}{1-\bar{\alpha}_t}\,\mathbf{x}_0.
\]

\medskip

DDIMs define an \emph{alternative} set of reverse conditionals that preserve the same single–step marginals $q(\mathbf{x}_t\mid \mathbf{x}_0)$ as DDPM, but allow a controlled per–step variance $\sigma_t^2$:
\begin{align}
\label{eq:ddimclosed}
q(\mathbf{x}_{t-1} \mid \mathbf{x}_t, \mathbf{x}_0)
&=
\mathcal{N}\!\Bigg(
\mathbf{x}_{t-1}\;;\;
\underbrace{\sqrt{\bar{\alpha}_{t-1}}\,\mathbf{x}_0
+
\sqrt{\,1-\bar{\alpha}_{t-1}-\sigma_t^2\,}\;
\frac{\mathbf{x}_t-\sqrt{\bar{\alpha}_t}\,\mathbf{x}_0}{\sqrt{1-\bar{\alpha}_t}}}_{\text{mean}},
\;\underbrace{\sigma_t^2}_{\text{variance}}\mathbf{I}_d\Bigg),\\
q(\mathbf{x}_T \mid \mathbf{x}_0)
&= \mathcal{N}\!\big(\mathbf{x}_T;\; \sqrt{\bar{\alpha}_T}\, \mathbf{x}_0,\; (1-\bar{\alpha}_T)\, \mathbf{I}_d\big).\nonumber
\end{align}
The mean in \eqref{eq:ddimclosed} is a convex combination (in a geometric, $\bar{\alpha}$–aware sense) of the two endpoints of the straight line
\[
\mathbf{x}_t
=
\sqrt{\bar{\alpha}_t}\,\mathbf{x}_0
+
\sqrt{1-\bar{\alpha}_t}\,\boldsymbol{\epsilon}.
\]
It keeps the part aligned with $\mathbf{x}_0$ as $\sqrt{\bar{\alpha}_{t-1}}\,\mathbf{x}_0$, and rescales the orthogonal (noise) component
\(
(\mathbf{x}_t-\sqrt{\bar{\alpha}_t}\,\mathbf{x}_0)/\sqrt{1-\bar{\alpha}_t}
\)
by a factor $\sqrt{\,1-\bar{\alpha}_{t-1}-\sigma_t^2\,}$ to match the target variance at time $t-1$.

\medskip
If we choose $\sigma_t^2=\tilde{\beta}_t:=\frac{1-\bar{\alpha}_{t-1}}{1-\bar{\alpha}_t}\,\beta_t$, then
\[
1-\bar{\alpha}_{t-1}-\sigma_t^2
=
(1-\bar{\alpha}_{t-1})\left(1-\frac{\beta_t}{1-\bar{\alpha}_t}\right)
=
(1-\bar{\alpha}_{t-1})\cdot \frac{\alpha_t(1-\bar{\alpha}_{t-1})}{1-\bar{\alpha}_t}
=
\frac{\alpha_t(1-\bar{\alpha}_{t-1})^2}{1-\bar{\alpha}_t}.
\]
Therefore
\[
\sqrt{\frac{\,1-\bar{\alpha}_{t-1}-\sigma_t^2\,}{\,1-\bar{\alpha}_t\,}}
=
\frac{\sqrt{\alpha_t}\,(1-\bar{\alpha}_{t-1})}{1-\bar{\alpha}_t},
\]
and the mean in \eqref{eq:ddimclosed} becomes
\[
\sqrt{\bar{\alpha}_{t-1}}\,\mathbf{x}_0
+
\frac{\sqrt{\alpha_t}(1-\bar{\alpha}_{t-1})}{1-\bar{\alpha}_t}\,
\big(\mathbf{x}_t-\sqrt{\bar{\alpha}_t}\,\mathbf{x}_0\big)
=
\frac{\sqrt{\alpha_t}(1-\bar{\alpha}_{t-1})}{1-\bar{\alpha}_t}\,\mathbf{x}_t
+
\frac{\sqrt{\bar{\alpha}_{t-1}}\beta_t}{1-\bar{\alpha}_t}\,\mathbf{x}_0,
\]
which is exactly the DDPM posterior mean shown above. Hence DDPM is the $\sigma_t^2=\tilde{\beta}_t$ member of the DDIM family.

\medskip
\begin{proposition}[DDIM preserves the DDPM marginals]
\label{prop:ddim_marginal_preserve}
For any choice of $\sigma_t^2\in[0,\,1-\bar{\alpha}_{t-1}]$ in \eqref{eq:ddimclosed}, the one–step marginal remains
\begin{equation}
\label{eq:lemma1}
q(\mathbf{x}_t \mid \mathbf{x}_0)
=
\mathcal{N}\!\big(\mathbf{x}_t;\; \sqrt{\bar{\alpha}_t}\, \mathbf{x}_0,\; (1-\bar{\alpha}_t)\, \mathbf{I}_d\big)\quad
\text{for every } t.
\end{equation}
\end{proposition}

\begin{proof}
We argue by backward induction on $t$. The base case $t=T$ is the second line of \eqref{eq:ddimclosed}. Assume the claim holds at time $t$. Consider the marginal at $t-1$:
\[
q(\mathbf{x}_{t-1}\mid \mathbf{x}_0)
=
\int q(\mathbf{x}_{t-1}\mid \mathbf{x}_t,\mathbf{x}_0)\;q(\mathbf{x}_t\mid \mathbf{x}_0)\,d\mathbf{x}_t,
\]
which is a Gaussian integral because both factors are Gaussian. Its mean is
\[
\boldsymbol{\mu}_{t-1}
=
\mathbb{E}_{q(\mathbf{x}_t\mid \mathbf{x}_0)}
\!\left[
\mathbb{E}_{q(\mathbf{x}_{t-1}\mid \mathbf{x}_t,\mathbf{x}_0)}[\mathbf{x}_{t-1}]
\right]
=
\mathbb{E}\!\left[
\sqrt{\bar{\alpha}_{t-1}}\,\mathbf{x}_0
+
\sqrt{1-\bar{\alpha}_{t-1}-\sigma_t^2}\;
\frac{\mathbf{x}_t-\sqrt{\bar{\alpha}_t}\,\mathbf{x}_0}{\sqrt{1-\bar{\alpha}_t}}
\right].
\]
Since $\mathbb{E}_{q(\mathbf{x}_t\mid \mathbf{x}_0)}[\mathbf{x}_t]=\sqrt{\bar{\alpha}_t}\,\mathbf{x}_0$ by the induction hypothesis, the second term vanishes in expectation, and we obtain
\(
\boldsymbol{\mu}_{t-1}=\sqrt{\bar{\alpha}_{t-1}}\,\mathbf{x}_0.
\)
For the covariance, use the law of total variance:
\[
\mathrm{Var}[\mathbf{x}_{t-1}\mid \mathbf{x}_0]
=
\underbrace{\mathbb{E}\big[\mathrm{Var}(\mathbf{x}_{t-1}\mid \mathbf{x}_t,\mathbf{x}_0)\big]}_{\displaystyle \sigma_t^2\,\mathbf{I}_d}
\;+\;
\underbrace{\mathrm{Var}\!\big(\mathbb{E}[\mathbf{x}_{t-1}\mid \mathbf{x}_t,\mathbf{x}_0]\big)}_{\displaystyle \big(\sqrt{1-\bar{\alpha}_{t-1}-\sigma_t^2}/\sqrt{1-\bar{\alpha}_t}\big)^{\!2}\ \mathrm{Var}(\mathbf{x}_t\mid \mathbf{x}_0)}.
\]
By the induction hypothesis, $\mathrm{Var}(\mathbf{x}_t\mid \mathbf{x}_0)=(1-\bar{\alpha}_t)\mathbf{I}_d$, so
\[
\mathrm{Var}[\mathbf{x}_{t-1}\mid \mathbf{x}_0]
=
\sigma_t^2\,\mathbf{I}_d
+
(1-\bar{\alpha}_{t-1}-\sigma_t^2)\,\mathbf{I}_d
=
(1-\bar{\alpha}_{t-1})\,\mathbf{I}_d.
\]
Hence $q(\mathbf{x}_{t-1}\mid \mathbf{x}_0)=\mathcal{N}\big(\sqrt{\bar{\alpha}_{t-1}}\mathbf{x}_0,\,(1-\bar{\alpha}_{t-1})\mathbf{I}_d\big)$, completing the induction.
\end{proof}

\medskip

During sampling we approximate $q(\mathbf{x}_{t-1}\mid \mathbf{x}_t,\mathbf{x}_0)$ with the model’s reverse kernel. Equation~\eqref{eq:ddimclosed} makes explicit that the degree of stochasticity is governed by $\sigma_t^2$:
\begin{itemize}
\item \emph{Deterministic ($\sigma_t^2=0$).} The update collapses to the mean:
\[
\mathbf{x}_{t-1}
=
\sqrt{\bar{\alpha}_{t-1}}\,\mathbf{x}_0
+
\sqrt{\frac{1-\bar{\alpha}_{t-1}}{1-\bar{\alpha}_t}}\,\big(\mathbf{x}_t-\sqrt{\bar{\alpha}_t}\,\mathbf{x}_0\big).
\]
Because $\mathbf{x}_0$ is unknown at sampling time, we use the standard estimator
\[
\hat{\mathbf{x}}_0(\mathbf{x}_t,t)
=
\frac{\mathbf{x}_t-\sqrt{1-\bar{\alpha}_t}\,\boldsymbol{\epsilon}_\theta(\mathbf{x}_t,t)}{\sqrt{\bar{\alpha}_t}},
\]
where $\boldsymbol{\epsilon}_\theta$ is the learned noise predictor. Substituting $\hat{\mathbf{x}}_0$ yields a purely deterministic sampler (the common “DDIM $\eta=0$” update).
\item \emph{Stochastic ($\sigma_t^2>0$).} We draw
\(
\mathbf{x}_{t-1}
=
\text{mean in \eqref{eq:ddimclosed}}
+
\sigma_t\,\mathbf{z},\ \mathbf{z}\sim\mathcal{N}(\mathbf{0},\mathbf{I}_d).
\)
Taking $\sigma_t^2=\tilde{\beta}_t$ gives back DDPM; smaller $\sigma_t^2$ tends to sharpen results while retaining some diversity.
\end{itemize}

\medskip

A practical benefit of DDIM is fast sampling on a reduced time grid. Let
\(
T=\tau_1>\tau_2>\cdots>\tau_S=1
\)
be a subset of indices. We apply \eqref{eq:ddimclosed} only on these steps:
\[
\mathbf{x}_{\tau_{k-1}}
\sim
q\!\left(\mathbf{x}_{\tau_{k-1}}\,\middle|\,\mathbf{x}_{\tau_k},\mathbf{x}_0\right),
\qquad k=S,\dots,2,
\]
using the same trained $\boldsymbol{\epsilon}_\theta$ and the $\hat{\mathbf{x}}_0$ estimator above. Proposition~\ref{prop:ddim_marginal_preserve} guarantees that $q(\mathbf{x}_{\tau_k}\mid \mathbf{x}_0)$ keeps the correct marginal
\(
\mathcal{N}\big(\sqrt{\bar{\alpha}_{\tau_k}}\mathbf{x}_0,\,(1-\bar{\alpha}_{\tau_k})\mathbf{I}_d\big)
\),
so the reduced chain remains consistent while enabling substantially fewer steps.

\subsection{DDGAN: Adversarially Learned Reverse Dynamics}
\label{subsec:ddgan}
\noindent
We describe an alternative training strategy that learns the reverse transition directly so that generation can proceed in far fewer iterations. We keep the forward process and notation from the Preliminaries: $\beta_t$, $\alpha_t=1-\beta_t$, $\bar{\alpha}_t=\prod_{s\le t}\alpha_s$, and
\[
q(\mathbf{x}_t\mid \mathbf{x}_0)=\mathcal{N}\!\big(\sqrt{\bar{\alpha}_t}\,\mathbf{x}_0,\ (1-\bar{\alpha}_t)\,\mathbf{I}_d\big),
\qquad
\mathbf{x}_t=\sqrt{\bar{\alpha}_t}\,\mathbf{x}_0+\sqrt{1-\bar{\alpha}_t}\,\boldsymbol{\varepsilon}.
\]

For each $t\in\{1,\dots,T\}$ we consider the distribution of consecutive latents under the forward chain,
\[
q_t^{\mathrm{pair}}(\mathbf{x}_{t-1},\mathbf{x}_t)
:= q(\mathbf{x}_t)\,q(\mathbf{x}_{t-1}\mid \mathbf{x}_t),
\qquad
q(\mathbf{x}_t)=\!\int q(\mathbf{x}_t\mid \mathbf{x}_0)\,p_{\mathrm{data}}(\mathbf{x}_0)\,d\mathbf{x}_0.
\]
A time–conditioned generator $G_\theta$ proposes a reverse step $\hat{\mathbf{x}}_{t-1}=G_\theta(\mathbf{x}_t,t,\mathbf{z})$ with $\mathbf{z}\sim\mathcal{N}(\mathbf{0},\mathbf{I}_d)$, while a discriminator $D_\phi(\,\cdot\,,\,\cdot\,,t)$ distinguishes real pairs $(\mathbf{x}_{t-1},\mathbf{x}_t)\sim q_t^{\mathrm{pair}}$ from generated pairs $(\hat{\mathbf{x}}_{t-1},\mathbf{x}_t)$. Parameters are typically shared across $t$ with explicit time embeddings.

\begin{definition}[Adversarial objective at discrete times]
\label{def:ddgan}
A simple logistic formulation trains $G_\theta$ and $D_\phi$ by
\[
\min_{\theta}\ \max_{\phi}\ \sum_{t=1}^{T}\Big\{
\mathbb{E}_{(\mathbf{x}_{t-1},\mathbf{x}_t)\sim q_t^{\mathrm{pair}}}\!\big[\log D_\phi(\mathbf{x}_{t-1},\mathbf{x}_t,t)\big]
+
\mathbb{E}_{\mathbf{x}_t\sim q(\mathbf{x}_t),\,\mathbf{z}\sim\mathcal{N}}\!
\big[\log\!\big(1-D_\phi(G_\theta(\mathbf{x}_t,t,\mathbf{z}),\mathbf{x}_t,t)\big)\big]
\Big\}.
\]
Other discriminator losses may be used; the conditioning on $t$ is unchanged.
\end{definition}

\begin{proposition}[Consistency of the learned reverse transition]
\label{prop:ddgan-consistency}
Fix any $t\in\{1,\dots,T\}$. Under the objective in Definition~\ref{def:ddgan}, if the discriminator at time $t$ is optimal, then the induced generator update moves the model joint distribution
\[
p_t^\theta(\mathbf{x}_{t-1},\mathbf{x}_t)\ :=\ q(\mathbf{x}_t)\,p_\theta(\mathbf{x}_{t-1}\mid \mathbf{x}_t)
\]
towards the real joint distribution
\[
q_t^{\mathrm{pair}}(\mathbf{x}_{t-1},\mathbf{x}_t)\ :=\ q(\mathbf{x}_t)\,q(\mathbf{x}_{t-1}\mid \mathbf{x}_t).
\]
In particular, the stationary point of the game at time $t$ is attained when $p_t^\theta=q_t^{\mathrm{pair}}$, i.e., the learned reverse transition reproduces the distribution of forward pairs at that time index.
\end{proposition}

\begin{proof}
With the logistic loss, the optimal discriminator at time $t$ is
\[
D_t^\star(\mathbf{x}_{t-1},\mathbf{x}_t)=
\frac{q_t^{\mathrm{pair}}(\mathbf{x}_{t-1},\mathbf{x}_t)}
     {q_t^{\mathrm{pair}}(\mathbf{x}_{t-1},\mathbf{x}_t)+p_t^\theta(\mathbf{x}_{t-1},\mathbf{x}_t)}.
\]
Substituting $D_t^\star$ into the per-$t$ objective yields
\[
\mathcal{L}_t(\theta)
= -\log 4\ +\ \mathrm{KL}\!\left(q_t^{\mathrm{pair}}\middle\|\tfrac{q_t^{\mathrm{pair}}+p_t^\theta}{2}\right)
+ \mathrm{KL}\!\left(p_t^\theta\middle\|\tfrac{q_t^{\mathrm{pair}}+p_t^\theta}{2}\right),
\]
after introducing $m=(q_t^{\mathrm{pair}}+p_t^\theta)/2$ and expanding elementary logarithms. Each KL term is nonnegative and vanishes if and only if its arguments coincide, hence $\mathcal{L}_t(\theta)\ge -\log 4$ with equality precisely when $p_t^\theta=q_t^{\mathrm{pair}}$. The generator update therefore decreases the discrepancy until the two joint distributions agree.
\end{proof}

\noindent
Sampling starts from $\mathbf{x}_T\sim\mathcal{N}(\mathbf{0},\mathbf{I}_d)$ and iterates $\mathbf{x}_{t-1}\leftarrow G_\theta(\mathbf{x}_t,t,\mathbf{z}_t)$ for a small number of steps. In practice one trains on a coarse grid $\{t_k\}_{k=K}^0$ with $K\ll T$ so that the learned reverse map supports large jumps; the same grid is then used at inference. It is common to output a residual relative to $\mathbf{x}_t$ (or in an equivalent parameterisation) rather than $\hat{\mathbf{x}}_{t-1}$ directly. Mild regularisation of $D_\phi$ improves stability, and small–noise consistency terms that encourage agreement with the Gaussian reverse mean used earlier may be included without changing the adversarial nature of training.

This approach qualifies as an acceleration method in the practical sense of wall-clock sampling (fewer steps), while it modifies the training criterion rather than the update rule of a likelihood-trained model. Conditioning for guidance typically enters through the inputs of $G_\theta$ (and, if desired, $D_\phi$).

\subsection{Nested Diffusion Models}
\label{subsec:nested_diffusion_models}

Nested diffusion models extend a single diffusion process by composing \emph{multiple}, hierarchically structured diffusion chains. Each level refines or augments a representation from the previous level, and together they define a joint generative model. We develop the construction from forward noising to reverse-time generation.

Let $\x_0\in\mathcal{X}$ be a data sample with density $p_{\text{data}}(\x_0)$. A single diffusion uses
\[
    q(\x_{1:T}\mid\x_0)=\prod_{t=1}^T q(\x_t\mid\x_{t-1}),\qquad
    p_\theta(\x_{0:T})=p(\x_T)\prod_{t=1}^T p_\theta(\x_{t-1}\mid\x_t),
\]
with Gaussian forward steps and a learned reverse model. Nested diffusion introduces $K$ levels, each with its own chain in a representation space, and links the levels through deterministic or learned transforms. For illustration, consider
\[
    \bigl(\x_0,\,\x_1,\,\ldots,\,\x_T^1\bigr)
    \;\rightarrow\;
    \bigl(\z_0,\,\z_1,\,\ldots,\,\z_T^2\bigr)
    \;\rightarrow\;\cdots\;\rightarrow\;
    \bigl(\mathbf{u}_0,\,\mathbf{u}_1,\,\ldots,\,\mathbf{u}_T^K\bigr),
\]
where $\x_0$ is the original data and $\mathbf{u}_T^K$ is the deepest noise state, with intermediate chains obtained, for example, by deterministic mappings
\[
    \z_0=f_1(\x_0),\qquad
    \mathbf{u}_0=f_2(\z_0),\qquad\ldots
\]
(e.g., downsampling, encoders, or other transforms). For level $k\in\{1,\ldots,K\}$, denote the state at step $t$ by $\y_t^k$ (so $\y_t^1\!=\!\x_t$, $\y_t^2\!=\!\z_t$, etc.). Each level has a forward process
\[
    q_k(\y_t^k\mid \y_{t-1}^k)
    \;=\;
    \mathcal{N}\!\Big(\y_t^k;\;\sqrt{\alpha_t^{(k)}}\,\y_{t-1}^k,\;\big(1-\alpha_t^{(k)}\big)\mathbf{I}\Big),
\quad
\alpha_t^{(k)}=1-\beta_t^{(k)},\quad \bar{\alpha}_t^{(k)}=\prod_{i=1}^t \alpha_i^{(k)}.
\]
The nested link is enforced via an initialisation rule such as $\y_0^k=g_k(\Y^{k-1})$ or $\y_0^k=g_k(\y_0^{k-1})$, where $\Y^{k-1}=(\y_0^{k-1},\ldots,\y_T^{k-1})$. 
\medskip

Let $\Y^k=(\y_0^k,\ldots,\y_T^k)$. The joint forward distribution factorises as
\[
    q(\Y^1,\ldots,\Y^K)
    \;=\;
    q_1(\Y^1)\,\prod_{k=2}^{K} q_k(\Y^k\mid \Y^{k-1}),
\]
with
\[
    q_1(\Y^1)=p_{\text{data}}(\x_0)\,\prod_{t=1}^T q_1(\x_t\mid \x_{t-1}),\qquad
    q_k(\Y^k\mid \Y^{k-1})
    = \delta\!\big(\y_0^k - g_k(\Y^{k-1})\big)\;\prod_{t=1}^T q_k(\y_t^k\mid \y_{t-1}^k).
\]
Here $\delta(\cdot)$ enforces the deterministic initialisation of level $k$ from level $k\!-\!1$. As in the single-level case, each level admits the closed-form marginal
\begin{equation}
\label{eq:nested_marginal}
    q_k(\y_t^k \mid \y_0^k)
    =
    \mathcal{N}\!\Big(\y_t^k;\;\sqrt{\bar{\alpha}_t^{(k)}}\,\y_0^k,\;\big(1-\bar{\alpha}_t^{(k)}\big)\mathbf{I}\Big),
\end{equation}
and the DDPM posterior
\begin{equation}
\label{eq:nested_posterior}
    q_k(\y_{t-1}^k \mid \y_t^k,\y_0^k)
    =
    \mathcal{N}\!\Bigg(
        \y_{t-1}^k;\;
        \frac{1}{\sqrt{\alpha_t^{(k)}}}\!\left(\y_t^k - \frac{\beta_t^{(k)}}{\sqrt{1-\bar{\alpha}_t^{(k)}}}\,\boldsymbol{\epsilon}_t^{(k)}\right),
        \;\tilde{\beta}_t^{(k)}\,\mathbf{I}
    \Bigg),
    \qquad
    \tilde{\beta}_t^{(k)} := \frac{1-\bar{\alpha}_{t-1}^{(k)}}{1-\bar{\alpha}_t^{(k)}}\,\beta_t^{(k)}.
\end{equation}

A nested reverse model generates jointly across levels. A common coarse-to-fine factorisation is
\begin{equation}
\label{eq:nested_reverse}
    p_\theta(\Y^1,\ldots,\Y^K)
    \;=\;
    \underbrace{p(\y_T^K)\prod_{t=1}^T p_\theta(\y_{t-1}^K \mid \y_t^K, K, t)}_{\text{deepest level}}
    \;\prod_{k=K-1}^{1}\;
    \underbrace{p_\theta(\y_0^k \mid \y_0^{k+1}, k)\;\prod_{t=1}^T p_\theta(\y_{t-1}^k \mid \y_t^k, \y_0^{k+1}, k, t)}_{\text{level }k\text{ conditioned on level }k\!+\!1},
\end{equation}
with $p(\y_T^K)=\mathcal{N}(\mathbf{0},\mathbf{I})$. The cross-level term $p_\theta(\y_0^k \mid \y_0^{k+1}, k)$ plays the role of an upsampler/decoder (deterministic or learned). The per-level reverse transitions are Gaussian, parameterised as in earlier sections via noise prediction:
\begin{align}
\label{eq:nested_mu_theta}
    p_\theta(\y_{t-1}^k \mid \y_t^k, \cdot)
    &= \mathcal{N}\!\Big(\y_{t-1}^k;\; \mu_\theta^{(k)}(\y_t^k,\cdot),\; \sigma_{t,k}^2\,\mathbf{I}\Big),\\
    \mu_\theta^{(k)}(\y_t^k,\cdot)
    &= \frac{1}{\sqrt{\alpha_t^{(k)}}}\!\left(\y_t^k - \frac{\beta_t^{(k)}}{\sqrt{1-\bar{\alpha}_t^{(k)}}}\;\hat{\boldsymbol{\epsilon}}^{(k)}_\theta(\y_t^k, t, k, \y_0^{k+1})\right),\nonumber
\end{align}
with $\sigma_{t,k}^2=\tilde{\beta}_t^{(k)}$ to recover the DDPM case, or other choices (e.g., DDIM-style) for accelerated samplers.

A typical generation path is:
\begin{enumerate}
    \item Sample $\y_T^K\sim\mathcal{N}(\mathbf{0},\mathbf{I})$; run reverse diffusion at level $K$ to obtain $\y_0^K$.
    \item Produce the next level’s condition via $\y_0^{K-1} \sim p_\theta(\y_0^{K-1}\mid \y_0^K,K\!-\!1)$ (or $\y_0^{K-1}=g_{K-1}^{-1}(\y_0^K)$ if deterministic).
    \item Initialise level $K\!-\!1$ at noise, $\y_T^{K-1}\sim\mathcal{N}(\mathbf{0},\mathbf{I})$, and run reverse diffusion with transitions conditioned on $\y_0^{K}$ to obtain $\y_0^{K-1}$.
    \item Repeat down to level $1$ to obtain $\x_0=\y_0^1$.
\end{enumerate}
Other orderings (fine-to-coarse) are possible; the key design choice is where the deterministic/learned cross-level mapping is placed (at the start or the end of each chain).

Training proceeds by minimising a variational bound on the negative log-likelihood:
\[
    \mathcal{L}
    \;=\;
    \mathrm{KL}\!\big(q(\Y^1,\ldots,\Y^K)\,\big\Vert\, p_\theta(\Y^1,\ldots,\Y^K)\big).
\]
As in the single-level case, this decomposes into sums of per-step KLs at each level, yielding practical mean-squared noise-matching terms:
\begin{equation}
\label{eq:nested_training}
    \mathcal{L}
    \;\cong\;
    \sum_{k=1}^{K}\sum_{t=2}^{T}
    w_{t,k}\,
    \mathbb{E}\!\left[
        \big\|
            \boldsymbol{\epsilon}^{(k)}_t
            - \hat{\boldsymbol{\epsilon}}^{(k)}_\theta(\y_t^k, t, k, \text{cond}_{k})
        \big\|_2^2
    \right],
    \qquad
    w_{t,k}=\frac{\beta_t^{(k)}}{2\big(1-\beta_t^{(k)}\big)\big(1-\bar{\alpha}_{t-1}^{(k)}\big)}.
\end{equation}
Here $\boldsymbol{\epsilon}^{(k)}_t$ is the true Gaussian noise injected at level $k$, and $\text{cond}_{k}$ denotes any cross-level conditioning (e.g., $\y_0^{k+1}$). The weights $w_{t,k}$ mirror earlier sections and ensure scale consistency across timesteps.

A convenient view is that each level handles a different \emph{scale} or \emph{latent} representation. For a two-level example, let
\[
    \mathbf{z}_0 = f(\x_0),
\]
with $f$ a deterministic downsampler/encoder. Run diffusion on the coarse space $q_2(\mathbf{z}_{1:T}\mid \mathbf{z}_0)$, reverse to obtain $\mathbf{z}_0$, lift $\mathbf{z}_0$ back towards the data space via a decoder (deterministic or learned), and then run/condition a second diffusion chain on $\x$:
\[
    q_1(\x_{1:T}\mid \x_0),\qquad
    q_2(\mathbf{z}_{1:T}\mid \mathbf{z}_0),\qquad
    \text{with cross-level links } \mathbf{z}_0 = f(\x_0)\;\text{ and }\; p_\theta(\x_0\mid \mathbf{z}_0).
\]
The overall forward joint still factorises as above, with each scale’s trajectory built upon the states of the previous scale.

The cross-level conditional $p_\theta(\y_0^k \mid \y_0^{k+1})$ can be deterministic (e.g., upsampling) or learned (e.g., a conditional decoder). After sampling/deriving $\y_0^k$, one either (i) sets $\y_T^k\!\sim\!\mathcal{N}(\mathbf{0},\mathbf{I})$ and runs the reverse chain $p_\theta(\y_{t-1}^k\mid \y_t^k,\y_0^{k+1})$ down to $\y_0^k$, or (ii) uses a DDIM-style reduced-step sampler at level $k$ with the same cumulative schedule $\bar{\alpha}_t^{(k)}$ for efficiency. In all cases, the parameterisation \eqref{eq:nested_mu_theta} with the cumulative $\bar{\alpha}_t^{(k)}$ keeps the notation and mechanics consistent with the single-level diffusion developed earlier.

\subsection{Stable Diffusion}
\label{subsec:stable-diffusion}

Stable Diffusion \cite{rombach-2022-ldm} is a \emph{latent diffusion} approach: rather than running diffusion directly in pixel space, the data are first mapped to a lower-dimensional latent space where denoising is cheaper; after sampling, a decoder maps the final latent back to the data domain. Let $\mathbf{x}_0\in\mathbb{R}^{H\times W\times C}$ be drawn from $p_{\text{data}}(\mathbf{x}_0)$. A trained encoder–decoder $(E,D)$ with $(h,w,c)\ll(H,W,C)$ defines
\[
E:\mathbb{R}^{H\times W\times C}\to\mathbb{R}^{h\times w\times c},\qquad
D:\mathbb{R}^{h\times w\times c}\to\mathbb{R}^{H\times W\times C},\qquad
\mathbf{z}_0 := E(\mathbf{x}_0),\quad \tilde{\mathbf{x}}_0 := D(\mathbf{z}_0).
\]
We treat $\mathbf{z}_0$ as the “data” for diffusion. When flattened, the latent lives in $\mathbb{R}^d$ with $d=h\,w\,c$. Using the same schedule notation as before, with $\alpha_t:=1-\beta_t$ and $\bar{\alpha}_t:=\prod_{i=1}^t\alpha_i$, the latent forward chain is linear–Gaussian,
\begin{equation}
\label{eq:ldm_forward}
q(\mathbf{z}_{1:T}\mid \mathbf{z}_0)=\prod_{t=1}^T q(\mathbf{z}_t\mid \mathbf{z}_{t-1}),
\qquad
q(\mathbf{z}_t\mid \mathbf{z}_{t-1})
=
\mathcal{N}\!\Big(\mathbf{z}_t;\ \sqrt{\alpha_t}\,\mathbf{z}_{t-1},\ \beta_t\,\mathbf{I}_d\Big),
\end{equation}
which implies the closed-form marginal and its re-parameterisation
\begin{equation}
\label{eq:ldm_closed}
q(\mathbf{z}_t\mid \mathbf{z}_0)
=
\mathcal{N}\!\Big(\mathbf{z}_t;\ \sqrt{\bar{\alpha}_t}\,\mathbf{z}_0,\ (1-\bar{\alpha}_t)\,\mathbf{I}_d\Big),
\qquad
\mathbf{z}_t=\sqrt{\bar{\alpha}_t}\,\mathbf{z}_0+\sqrt{1-\bar{\alpha}_t}\,\boldsymbol{\epsilon},
\quad \boldsymbol{\epsilon}\sim\mathcal{N}(\mathbf{0},\mathbf{I}_d).
\end{equation}
Because \eqref{eq:ldm_forward} is linear–Gaussian, the exact single-step posterior mirrors the pixel-space DDPM posterior, now in latent space:
\begin{equation}
\label{eq:ldm_posterior}
q(\mathbf{z}_{t-1}\mid \mathbf{z}_t,\mathbf{z}_0)
=
\mathcal{N}\!\Bigg(
\mathbf{z}_{t-1}\ ;\
\frac{1}{\sqrt{\alpha_t}}\!\left(\mathbf{z}_t-\frac{\beta_t}{\sqrt{1-\bar{\alpha}_t}}\,\boldsymbol{\epsilon}_t\right),
\ \tilde{\beta}_t\,\mathbf{I}_d
\Bigg),
\qquad
\tilde{\beta}_t:=\frac{1-\bar{\alpha}_{t-1}}{1-\bar{\alpha}_t}\,\beta_t,
\quad
\boldsymbol{\epsilon}_t=\frac{\mathbf{z}_t-\sqrt{\bar{\alpha}_t}\,\mathbf{z}_0}{\sqrt{1-\bar{\alpha}_t}}.
\end{equation}

Generation samples a Gaussian prior at the terminal noise level and reverses the chain with a learned mean (via noise prediction) and a chosen per-step variance,
\[
p_\theta(\mathbf{z}_{0:T}) \;=\; p(\mathbf{z}_T)\prod_{t=1}^T p_\theta(\mathbf{z}_{t-1}\mid \mathbf{z}_t),
\qquad
p(\mathbf{z}_T)=\mathcal{N}(\mathbf{0},\mathbf{I}_d),
\]
\begin{equation}
\label{eq:ldm_reverse}
p_\theta(\mathbf{z}_{t-1}\mid \mathbf{z}_t)
=
\mathcal{N}\!\Bigg(
\mathbf{z}_{t-1}\ ;\
\underbrace{\frac{1}{\sqrt{\alpha_t}}
\left(\mathbf{z}_t-\frac{\beta_t}{\sqrt{1-\bar{\alpha}_t}}\,\hat{\boldsymbol{\epsilon}}_\theta(\mathbf{z}_t,t)\right)}_{\mu_\theta(\mathbf{z}_t,t)},
\ \sigma_t^2\,\mathbf{I}_d
\Bigg),
\qquad
\sigma_t^2\in\{\tilde{\beta}_t,\ 0\}\ \text{(DDPM or DDIM sampling)}.
\end{equation}
This is exactly the parameterisation from earlier sections after the replacement $\mathbf{x}\mapsto \mathbf{z}$. If conditioning is used (for example, text), it simply augments $\hat{\boldsymbol{\epsilon}}_\theta$ as $\hat{\boldsymbol{\epsilon}}_\theta(\mathbf{z}_t,t,\cdot)$; the mechanisms are developed later in Section~\ref{sec:guided_diffusion}.

Training in latent space follows the same ELBO/KL decomposition and yields weighted noise matching with the forward reparameterisation from \eqref{eq:ldm_closed}:
\begin{equation}
\label{eq:ldm_loss}
\mathcal{L}_{\text{LDM}}
\;\cong\;
\mathbb{E}_{\mathbf{z}_0,t,\boldsymbol{\epsilon}}
\left[
w_t\,\big\|
\boldsymbol{\epsilon}
-\hat{\boldsymbol{\epsilon}}_\theta\!\big(\mathbf{z}_t,t\big)
\big\|_2^2
\right],
\qquad
\mathbf{z}_t=\sqrt{\bar{\alpha}_t}\,\mathbf{z}_0+\sqrt{1-\bar{\alpha}_t}\,\boldsymbol{\epsilon},
\quad
w_t=\frac{\beta_t}{2(1-\beta_t)(1-\bar{\alpha}_{t-1})}.
\end{equation}
Any equivalent re-weighting that leaves the Bayes-optimal predictor unchanged is acceptable; the form above is shown for consistency with the earlier pixel-space derivation.

After the reverse process yields $\widehat{\mathbf{z}}_0$, the decoder reconstructs the output in data space,
\begin{equation}
\label{eq:ldm_decode}
\widehat{\mathbf{x}}_0 \;=\; D\!\left(\widehat{\mathbf{z}}_0\right).
\end{equation}
Operating in latent space concentrates computation on semantically structured features learned by $(E,D)$, preserving the diffusion mathematics (forward marginals, true posterior, reverse parameterisation) while substantially reducing sampling cost.

\section{Flow Matching}

Flow matching reframes generative modelling as learning a time–dependent velocity field that \emph{deterministically} transports probability mass from a trivial source distribution to the data distribution. When discussing diffusion models, we specified a family of intermediate noisy marginals $\{p_t\}_{t\in[0,T]}$ (discrete schedules via $\alpha_t$ and $\bar{\alpha}_t$) and then sampled by reversing a stochastic forward process. Flow matching asks for a velocity field $\mathbf{v}(\mathbf{x},t)$ whose Ordinary Differential Equation (ODE)
\[
\dot{\mathbf{X}}_t=\mathbf{v}(\mathbf{X}_t,t)
\]
has \emph{exactly the same time marginals} $\{p_t\}$ as those defined by the diffusion construction \cite{lipman-2023-flowmatching}. In other words, we seek a deterministic transport that is marginally indistinguishable (at each $t$) from the stochastic diffusion.\medskip

Flow matching is useful for two reasons. First, replacing stepwise stochastic simulation with deterministic ODE integration removes trajectory noise and permits the use of standard high-order ODE solvers. Second, once a velocity field $\mathbf{v}$ is specified, it can be trained to reproduce the prescribed marginals ${p_t}_{t\in[0,T]}$ directly, without committing to a particular stochastic forward mechanism during training.\medskip

To make this precise we pass through two standard pieces of analysis. The first is the continuity (Liouville) equation, which encodes conservation of probability under a time–dependent velocity field and characterises the marginals of the ODE above. The second is the Fokker–Planck equation for the forward diffusion; it characterises the evolution of the diffusion marginals. By rewriting the diffusion’s second-order (Laplacian) term in the Fokker–Planck equation as a conservative transport driven by the score $s_t=\nabla_{\mathbf{x}}\log p_t$, we derive a \emph{probability–flow ODE} whose velocity
\[
\mathbf{v}(\mathbf{x},t)=\mathbf{f}(\mathbf{x},t)-\tfrac{1}{2}g(t)^2\,s_t(\mathbf{x})
\]
produces the same marginals as the diffusion $\mathrm{d}\mathbf{X}_t=\mathbf{f}(\mathbf{X}_t,t)\,\mathrm{d}t+g(t)\,\mathrm{d}\mathbf{W}_t$. In the variance-preserving case used earlier, this recovers the familiar DDIM form when the true score is replaced by its learned approximation.

The next subsection states these ingredients carefully and proves the probability–flow identity we will use throughout the section.

\subsection{Probability--Flow ODE and the Continuity Equation}

A time-indexed family of densities $(p_t)_{t\in[0,T]}$ and a velocity field $\mathbf{v}(\cdot,t)$ are linked by the continuity equation
\[
\partial_t p_t(\mathbf{x}) \;=\; -\,\nabla_{\mathbf{x}}\!\cdot\!\big(p_t(\mathbf{x})\,\mathbf{v}(\mathbf{x},t)\big),
\]
which expresses conservation of probability along the flow generated by $\mathbf{v}$. The first step is to connect deterministic ODE trajectories with this Partial Differential Equation (PDE) \cite{song-2021-sde}.

\begin{definition}[Continuity (Liouville) equation]
A pair $(p_t,\mathbf{v})$ satisfies the continuity equation on $[0,T]$ if
\[
\partial_t p_t \;=\; -\,\nabla_{\mathbf{x}}\!\cdot\!\big(p_t\,\mathbf{v}\big)
\quad\text{with}\quad
p_{t=0}=p_0.
\]
\end{definition}

\begin{lemma}[Liouville transport for ODE flows]\label{lem:liouville}
Let $\dot{\mathbf{X}}_t=\mathbf{v}(\mathbf{X}_t,t)$ with $\mathbf{X}_0\sim p_0$, where $\mathbf{v}$ is locally Lipschitz in $\mathbf{x}$ and measurable in $t$, and solutions exist uniquely without explosion on $[0,T]$. Then the law $p_t$ of $\mathbf{X}_t$ satisfies the continuity equation with velocity $\mathbf{v}$.
\end{lemma}

\begin{proof}
Fix smooth compactly supported $\varphi$. By the chain rule,
\(
\frac{\mathrm{d}}{\mathrm{d}t}\varphi(\mathbf{X}_t)=\nabla\varphi(\mathbf{X}_t)^\top \mathbf{v}(\mathbf{X}_t,t).
\)
Taking expectations and using the law of $\mathbf{X}_t$,
\[
\frac{\mathrm{d}}{\mathrm{d}t}\int \varphi(\mathbf{x})\,p_t(\mathbf{x})\,\mathrm{d}\mathbf{x}
=
\int \nabla\varphi(\mathbf{x})^\top \mathbf{v}(\mathbf{x},t)\,p_t(\mathbf{x})\,\mathrm{d}\mathbf{x}.
\]
Integrate by parts in $\mathbf{x}$ (boundary term vanishes by compact support):
\[
\int \nabla\varphi^\top \mathbf{v}\,p_t
=
-\,\int \varphi(\mathbf{x})\,\nabla_{\mathbf{x}}\!\cdot\!\big(p_t(\mathbf{x})\,\mathbf{v}(\mathbf{x},t)\big)\,\mathrm{d}\mathbf{x}.
\]
Since this holds for all $\varphi$, we obtain $\partial_t p_t=-\nabla_{\mathbf{x}}\!\cdot(p_t\mathbf{v})$ in the distributional (hence classical) sense.
\end{proof}

We now recall the Fokker–Planck equation for the forward SDE used in diffusion models and show how it connects to the continuity equation.

\begin{lemma}[Fokker--Planck for isotropic, state–independent diffusion]\label{lem:fp}
Let $\mathrm{d}\mathbf{X}_t=\mathbf{f}(\mathbf{X}_t,t)\,\mathrm{d}t + g(t)\,\mathrm{d}\mathbf{W}_t$, with $\mathbf{f}$ measurable and of linear growth, $g$ continuous, and assume $\mathbf{X}_t$ has a smooth density $p_t$. Then
\[
\partial_t p_t(\mathbf{x}) \;=\; -\,\nabla_{\mathbf{x}}\!\cdot\!\big(\mathbf{f}(\mathbf{x},t)\,p_t(\mathbf{x})\big)
\;+\; \frac{g(t)^2}{2}\,\Delta_{\mathbf{x}} p_t(\mathbf{x}).
\]
\end{lemma}

\begin{proof}
Apply Itô’s formula to $\varphi(\mathbf{X}_t)$ for smooth compactly supported $\varphi$:
\[
\mathrm{d}\varphi(\mathbf{X}_t)
=
\nabla\varphi(\mathbf{X}_t)^\top \mathbf{f}(\mathbf{X}_t,t)\,\mathrm{d}t
+\tfrac{1}{2}\,\mathrm{Tr}\!\big(g(t)^2\,\nabla^2\varphi(\mathbf{X}_t)\big)\,\mathrm{d}t
+\nabla\varphi(\mathbf{X}_t)^\top g(t)\,\mathrm{d}\mathbf{W}_t.
\]
Take expectations; the martingale term has zero mean. Writing the expectation via $p_t$,
\[
\frac{\mathrm{d}}{\mathrm{d}t}\int \varphi\,p_t
=
\int \nabla\varphi^\top \mathbf{f}\,p_t \,\mathrm{d}\mathbf{x}
+ \frac{g(t)^2}{2}\int \Delta \varphi\,p_t \,\mathrm{d}\mathbf{x}.
\]
Integrate by parts in $\mathbf{x}$ in both terms to move derivatives onto $p_t$:
\[
\frac{\mathrm{d}}{\mathrm{d}t}\int \varphi\,p_t
=
- \int \varphi\,\nabla\!\cdot(\mathbf{f}p_t)\,\mathrm{d}\mathbf{x}
+ \frac{g(t)^2}{2}\int \varphi\,\Delta p_t\,\mathrm{d}\mathbf{x}.
\]
Since this holds for all $\varphi$, the PDE stated follows.
\end{proof}

The probability–flow ODE is obtained by rewriting the diffusion (second-order) term in Lemma~\ref{lem:fp} as a conservative transport term using the score $s_t=\nabla_{\mathbf{x}}\log p_t$.

\begin{theorem}[Probability--flow ODE (isotropic, state–independent diffusion)]\label{thm:pflow}
Define
\[
\mathbf{v}(\mathbf{x},t)
\;:=\;
\mathbf{f}(\mathbf{x},t) \;-\; \frac{g(t)^2}{2}\,s_t(\mathbf{x})
\;=\;
\mathbf{f}(\mathbf{x},t) \;-\; \frac{g(t)^2}{2}\,\nabla_{\mathbf{x}}\log p_t(\mathbf{x}).
\]
Let $\dot{\mathbf{Y}}_t=\mathbf{v}(\mathbf{Y}_t,t)$ with $\mathbf{Y}_0\sim p_0$. Then the time marginals of $\mathbf{Y}_t$ equal those of the SDE, i.e., $\mathbf{Y}_t\sim p_t$ for all $t\in[0,T]$.
\end{theorem}

\begin{proof}
From Lemma~\ref{lem:fp}, $\partial_t p_t = -\nabla\!\cdot(\mathbf{f}p_t) + \frac{g^2}{2}\Delta p_t$. Use the identity
\[
\Delta p_t
=
\nabla\!\cdot(\nabla p_t)
=
\nabla\!\cdot\!\big(p_t\,\nabla\log p_t\big),
\]
which is immediate from $\nabla\log p_t=(\nabla p_t)/p_t$. Hence
\[
\partial_t p_t
=
-\,\nabla\!\cdot\!\Big(\mathbf{f}\,p_t - \tfrac{g^2}{2}\,p_t\,\nabla\log p_t\Big)
=
-\,\nabla\!\cdot\!\big(p_t\,\mathbf{v}\big).
\]
Thus $p_t$ solves the continuity equation with velocity $\mathbf{v}$. By Lemma~\ref{lem:liouville}, the law of the ODE solution also solves the same continuity equation with the same initial condition, whence the marginals coincide.
\end{proof}

The diffusion models used earlier are recovered as a special case. In the variance–preserving (VP) forward process,
\[
\mathrm{d}\mathbf{X}_t = -\tfrac{1}{2}\,\beta(t)\,\mathbf{X}_t\,\mathrm{d}t + \sqrt{\beta(t)}\,\mathrm{d}\mathbf{W}_t,
\]
we have $\mathbf{f}(\mathbf{x},t)=-\tfrac{1}{2}\beta(t)\mathbf{x}$ and $g(t)^2=\beta(t)$, so Theorem~\ref{thm:pflow} yields
\[
\dot{\mathbf{x}}
\;=\;
-\tfrac{1}{2}\,\beta(t)\,\mathbf{x}
\;-\;
\tfrac{1}{2}\,\beta(t)\,\nabla_{\mathbf{x}}\log p_t(\mathbf{x}),
\]
which is the DDIM probability–flow ODE once the score is replaced by its learned approximation.

\begin{remark}[Scope and extensions]
The argument above assumes state–independent, isotropic diffusion $g(t)\mathbf{I}_d$, which is the regime adopted throughout this tutorial. With state–dependent diffusion matrices, the Fokker–Planck operator becomes $\partial_t p_t=-\nabla\!\cdot(\mathbf{f}p_t)+\tfrac{1}{2}\nabla\!\cdot(\mathbf{D}\nabla p_t)$, and the corresponding probability–flow velocity acquires additional terms beyond $-\tfrac{1}{2}\mathbf{D}\nabla\log p_t$. These cases require extra regularity to guarantee well-posedness. In practice, the standard VP/VE formulations avoid these complications and the ODE obtained above provides the deterministic counterpart whose time marginals match those of the diffusion.
\end{remark}

\subsection{Flow-Matching Objectives: Marginal vs.\ Conditional}

The probability–flow ODE of Theorem~\ref{thm:pflow} shows that a time–dependent velocity field $\mathbf{v}(\mathbf{x},t)$ uniquely determines an evolution of marginals via the continuity equation. \emph{Flow matching} turns this into a supervised learning problem: we construct a family of intermediate distributions $\{p_t\}_{t\in[0,T]}$ together with a target velocity field $\mathbf{u}_t(\mathbf{x})$ that transports $p_t$, and train a model $\mathbf{v}_\theta(\mathbf{x},t)$ to regress onto $\mathbf{u}_t(\mathbf{x})$. This avoids explicit score estimation; it requires only samples from $p_0$ (data) and from a simple terminal $p_T$ (e.g., $\mathcal{N}(\mathbf{0},\mathbf{I}_d)$), along with a differentiable path interpolating between endpoints.

\begin{definition}[Stochastic interpolant and coupling]\label{def:interpolant}
Let $p_0$ and $p_T$ be probability densities on $\mathbb{R}^d$. A \emph{coupling} is a joint law $\pi(\mathbf{x}_0,\mathbf{x}_T)$ with marginals $p_0$ and $p_T$. A \emph{stochastic interpolant} is a measurable map
\[
\psi:\ \mathbb{R}^d\times\mathbb{R}^d\times[0,T]\to\mathbb{R}^d,\qquad
(\mathbf{x}_0,\mathbf{x}_T,t)\mapsto \psi_t(\mathbf{x}_0,\mathbf{x}_T),
\]
such that $\psi_0(\mathbf{x}_0,\mathbf{x}_T)=\mathbf{x}_0$, $\psi_T(\mathbf{x}_0,\mathbf{x}_T)=\mathbf{x}_T$, and $t\mapsto \psi_t(\mathbf{x}_0,\mathbf{x}_T)$ is continuously differentiable. With
\[
\mathbf{X}_0\sim p_0,\quad \mathbf{X}_T\sim p_T,\quad
(\mathbf{X}_0,\mathbf{X}_T)\sim \pi,\quad
\mathbf{X}_t := \psi_t(\mathbf{X}_0,\mathbf{X}_T),\quad
p_t = \text{\rm law}(\mathbf{X}_t),
\]
write the conditional (per–path) velocity as $\dot{\psi}_t(\mathbf{x}_0,\mathbf{x}_T):=\partial_t \psi_t(\mathbf{x}_0,\mathbf{x}_T)$.
\end{definition}

\begin{definition}[Marginal target velocity]\label{def:marginal-velocity}
For $t\in[0,T]$, define
\[
\mathbf{u}_t(\mathbf{x}) \;:=\; \mathbb{E}\!\left[\,\dot{\psi}_t(\mathbf{X}_0,\mathbf{X}_T)\ \middle|\ \mathbf{X}_t=\mathbf{x}\,\right].
\]
\end{definition}

\begin{proposition}[Interpolant continuity equation]\label{prop:interpolant-continuity}
Under the regularity in Definition~\ref{def:interpolant}, the marginals $p_t$ satisfy
\[
\partial_t p_t(\mathbf{x}) \;=\; -\,\nabla_{\mathbf{x}}\!\cdot\!\big(p_t(\mathbf{x})\,\mathbf{u}_t(\mathbf{x})\big),\qquad t\in[0,T].
\]
Consequently, the ODE $\dot{\mathbf{Y}}_t=\mathbf{u}_t(\mathbf{Y}_t)$ with $\mathbf{Y}_0\sim p_0$ has $\text{\rm law}(\mathbf{Y}_t)=p_t$ for all $t$.
\end{proposition}

\begin{proof}
Let $\varphi$ be smooth with compact support. By the chain rule,
\(
\frac{\mathrm{d}}{\mathrm{d}t}\varphi(\mathbf{X}_t)
= \nabla\varphi(\mathbf{X}_t)^\top \dot{\psi}_t(\mathbf{X}_0,\mathbf{X}_T).
\)
Taking expectations and conditioning on $\mathbf{X}_t$ yields
\[
\frac{\mathrm{d}}{\mathrm{d}t}\int \varphi\,p_t
= \int \nabla\varphi(\mathbf{x})^\top \mathbf{u}_t(\mathbf{x})\,p_t(\mathbf{x})\,\mathrm{d}\mathbf{x}.
\]
An integration by parts gives $\int \nabla\varphi^\top \mathbf{u}_t\,p_t = -\int \varphi\,\nabla\!\cdot(p_t\mathbf{u}_t)$, hence the claim; the ODE marginal statement follows from Lemma~\ref{lem:liouville}.
\end{proof}

Two regression objectives are natural. The \emph{Conditional Flow–Matching} (CFM) objective trains on per–path targets available from the interpolant:
\[
\mathcal{L}_{\mathrm{CFM}}(\theta)
\;:=\; \mathbb{E}\!\left[\,\big\|\mathbf{v}_\theta(\mathbf{Z}_t,t)-\mathbf{U}_t\big\|_2^2\,\right],
\qquad
\mathbf{Z}_t:=\psi_t(\mathbf{X}_0,\mathbf{X}_T),\ \ \mathbf{U}_t:=\dot{\psi}_t(\mathbf{X}_0,\mathbf{X}_T).
\]
When an analytic form for the marginal velocity is available, the \emph{Marginal Flow–Matching} (MFM) objective regresses directly on $\mathbf{u}_t$:
\[
\mathcal{L}_{\mathrm{MFM}}(\theta)
\;:=\; \mathbb{E}\!\left[\,\big\|\mathbf{v}_\theta(\mathbf{X}_t,t)-\mathbf{u}_t(\mathbf{X}_t)\big\|_2^2\,\right],
\qquad \mathbf{X}_t\sim p_t.
\]

\begin{theorem}[Optimal predictor and CFM/MFM equivalence]\label{thm:cfm-mfm-equivalence}
Assume $\mathbb{E}\|\mathbf{U}_t\|_2^2<\infty$ and $\mathbb{E}\|\mathbf{v}_\theta(\mathbf{Z}_t,t)\|_2^2<\infty$ for all $\theta$. Then:
\begin{enumerate}
\item The unique $L^2$ minimiser of $\mathcal{L}_{\mathrm{CFM}}$ over measurable $\mathbf{v}$ is
\(
\mathbf{v}^*(\mathbf{x},t) \equiv \mathbf{u}_t(\mathbf{x})
= \mathbb{E}[\mathbf{U}_t\mid \mathbf{Z}_t=\mathbf{x}].
\)
\item There exists a constant $C$ independent of $\theta$ such that
\(
\mathcal{L}_{\mathrm{CFM}}(\theta)
= \mathcal{L}_{\mathrm{MFM}}(\theta) + C,
\)
so CFM and MFM share the same set of global minimisers.
\end{enumerate}
\end{theorem}

\begin{proof}
Let $Y:=(\mathbf{Z}_t,t)$ and $U:=\mathbf{U}_t$. Then $\mathcal{L}_{\mathrm{CFM}}(\theta)=\mathbb{E}\| \mathbf{v}_\theta(Y)-U\|_2^2$. By the $L^2$ orthogonality principle,
\(
\arg\min_{\mathbf{v}}\mathbb{E}\|\mathbf{v}(Y)-U\|_2^2
=\mathbb{E}[U\mid Y]
=\mathbf{u}_t(\mathbf{Z}_t).
\)
For (2), apply the law of total variance:
\[
\mathbb{E}\| \mathbf{v}_\theta(Y)-U\|_2^2
= \mathbb{E}\| \mathbf{v}_\theta(Y)-\mathbb{E}[U\mid Y]\|_2^2
+ \mathbb{E}\| U-\mathbb{E}[U\mid Y]\|_2^2,
\]
where the second term is constant in $\theta$ and the first equals $\mathcal{L}_{\mathrm{MFM}}(\theta)$ with target $\mathbf{u}_t$.
\end{proof}

In practice, CFM is attractive because for common interpolants the analytic $\mathbf{u}_t(\mathbf{x})$ involves the intractable posterior $\pi(\mathbf{x}_0,\mathbf{x}_T\mid \mathbf{X}_t=\mathbf{x})$, whereas $\dot{\psi}_t(\mathbf{X}_0,\mathbf{X}_T)$ is directly observable from sampled endpoints. The continuity equation determines $\mathbf{u}_t$ only up to fields $\mathbf{w}$ satisfying $\nabla\!\cdot(p_t\mathbf{w})=0$; CFM implicitly fixes an $L^2(p_t)$ representative induced by the chosen interpolant and coupling. Any positive time weighting $w(t)$ on $[0,T]$ can be used inside the expectation without changing the optimal predictor; weights are chosen for optimisation stability rather than identifiability.

As a concrete instance that will be studied further in the next subsection, the straight–line path $\psi_t=(1-\rho(t))\mathbf{X}_0+\rho(t)\mathbf{X}_T$ (with smooth $\rho(0)=0$, $\rho(T)=1$) yields the observable per–path target
\[
\mathbf{U}_t=\dot{\psi}_t(\mathbf{X}_0,\mathbf{X}_T)=\dot{\rho}(t)\,(\mathbf{X}_T-\mathbf{X}_0),
\]
so CFM trains $\mathbf{v}_\theta(\psi_t(\mathbf{X}_0,\mathbf{X}_T),t)$ to match $\dot{\rho}(t)(\mathbf{X}_T-\mathbf{X}_0)$ using only samples of $(\mathbf{X}_0,\mathbf{X}_T)$.

\subsection{Linear and Rectified Flows (Straight–Line Couplings)}

This subsection specialises the flow–matching construction to straight–line interpolants between data and Gaussian noise. We obtain a closed–form marginal velocity in terms of the score $\nabla_{\mathbf{x}}\log p_t(\mathbf{x})$ via a scaled Tweedie identity, and we isolate a time–only factor that motivates a \emph{rectified} velocity field with better numerical behaviour. Throughout, $\mathbf{X}_0\sim p_0$ is independent of $\mathbf{Z}\sim\mathcal{N}(\mathbf{0},\mathbf{I}_d)$, $\rho:[0,T]\to[0,1]$ is $C^1$ with $\rho(0)=0$, $\rho(T)=1$, and
\[
\mathbf{X}_t \;=\; (1-\rho(t))\,\mathbf{X}_0 \;+\; \rho(t)\,\mathbf{Z}, 
\qquad p_t=\text{\rm law}(\mathbf{X}_t).
\]
The conditional flow–matching (CFM) target \cite{albergo-2023-stochastic-interpolants} from Definition~\ref{def:interpolant} is

\[
\dot{\psi}_t(\mathbf{X}_0,\mathbf{Z})=\dot{\rho}(t)\,(\mathbf{Z}-\mathbf{X}_0).
\]

Our goal is the corresponding \emph{marginal} velocity $\mathbf{u}_t(\mathbf{x})=\mathbb{E}[\dot{\psi}_t(\mathbf{X}_0,\mathbf{Z})\mid \mathbf{X}_t=\mathbf{x}]$.

We first record a scaled Tweedie identity for additive Gaussian smoothing.

\begin{lemma}[Scaled Tweedie identity]\label{lem:scaled-tweedie}
Let $\mathbf{Y}=\mathbf{U}+\sigma\mathbf{Z}$ with $\mathbf{U}\in\mathbb{R}^d$ independent of $\mathbf{Z}\sim\mathcal{N}(\mathbf{0},\mathbf{I}_d)$ and $\sigma>0$, and let $p_{\mathbf{Y}}$ be the density of $\mathbf{Y}$. Then for all $\mathbf{y}$ where $p_{\mathbf{Y}}$ is differentiable,
\[
\mathbb{E}[\mathbf{U}\mid \mathbf{Y}=\mathbf{y}] \;=\; \mathbf{y} + \sigma^2 \nabla_{\mathbf{y}}\log p_{\mathbf{Y}}(\mathbf{y}),
\qquad
\mathbb{E}[\mathbf{Z}\mid \mathbf{Y}=\mathbf{y}] \;=\; -\,\sigma\,\nabla_{\mathbf{y}}\log p_{\mathbf{Y}}(\mathbf{y}).
\]
\end{lemma}

\begin{proof}
Write $p_{\mathbf{Y}}(\mathbf{y})=\int p_{\mathbf{U}}(\mathbf{u})\,\phi_\sigma(\mathbf{y}-\mathbf{u})\,\mathrm{d}\mathbf{u}$ with $\phi_\sigma$ the $\mathcal{N}(\mathbf{0},\sigma^2\mathbf{I}_d)$ density. Then
\[
\nabla_{\mathbf{y}} p_{\mathbf{Y}}(\mathbf{y})
= \int p_{\mathbf{U}}(\mathbf{u})\,\nabla_{\mathbf{y}}\phi_\sigma(\mathbf{y}-\mathbf{u})\,\mathrm{d}\mathbf{u}
= -\frac{1}{\sigma^2}\int (\mathbf{y}-\mathbf{u})\,p_{\mathbf{U}}(\mathbf{u})\,\phi_\sigma(\mathbf{y}-\mathbf{u})\,\mathrm{d}\mathbf{u}.
\]
Divide by $p_{\mathbf{Y}}(\mathbf{y})$ to obtain
\(
\nabla_{\mathbf{y}}\log p_{\mathbf{Y}}(\mathbf{y})
= -\frac{1}{\sigma^2}\big(\mathbf{y}-\mathbb{E}[\mathbf{U}\mid \mathbf{Y}=\mathbf{y}]\big).
\)
This yields the first identity. Since $\mathbf{Z}=(\mathbf{Y}-\mathbf{U})/\sigma$, the second follows by taking conditional expectations and substituting the first.
\end{proof}

We now obtain a closed–form marginal velocity for straight–line couplings.

\begin{proposition}[Straight–line marginal velocity]\label{prop:straightline-u}
With $\mathbf{X}_t=(1-\rho)\mathbf{X}_0+\rho\mathbf{Z}$ and $\dot{\psi}_t=\dot{\rho}(\mathbf{Z}-\mathbf{X}_0)$ as above,
\[
\mathbf{u}_t(\mathbf{x})
\;=\; \mathbb{E}[\dot{\psi}_t\mid \mathbf{X}_t=\mathbf{x}]
\;=\; -\,\frac{\dot{\rho}(t)}{1-\rho(t)}\Big(\mathbf{x} + \rho(t)\,\nabla_{\mathbf{x}}\log p_t(\mathbf{x})\Big).
\]
\end{proposition}

\begin{proof}
Set $\mathbf{U}=(1-\rho)\mathbf{X}_0$, $\sigma=\rho$, and $\mathbf{Y}=\mathbf{X}_t=\mathbf{U}+\sigma\mathbf{Z}$. By Lemma~\ref{lem:scaled-tweedie},
\[
\mathbb{E}[\mathbf{X}_0\mid \mathbf{X}_t=\mathbf{x}]
= \frac{1}{1-\rho}\,\mathbb{E}[\mathbf{U}\mid \mathbf{Y}=\mathbf{x}]
= \frac{1}{1-\rho}\big(\mathbf{x}+\rho^2\nabla_{\mathbf{x}}\log p_t(\mathbf{x})\big),
\]
and
\(
\mathbb{E}[\mathbf{Z}\mid \mathbf{X}_t=\mathbf{x}] = -\,\rho\,\nabla_{\mathbf{x}}\log p_t(\mathbf{x}).
\)
Therefore
\[
\mathbf{u}_t(\mathbf{x})
= \dot{\rho}\Big(\mathbb{E}[\mathbf{Z}\mid \mathbf{X}_t=\mathbf{x}] - \mathbb{E}[\mathbf{X}_0\mid \mathbf{X}_t=\mathbf{x}]\Big)
= \dot{\rho}\Big(-\rho\nabla\log p_t - \frac{1}{1-\rho}(\mathbf{x}+\rho^2\nabla\log p_t)\Big),
\]
which simplifies to the stated expression.
\end{proof}

Proposition~\ref{prop:straightline-u} shows that straight–line flow matching produces a velocity of the form
\[
\mathbf{u}_t(\mathbf{x}) \;=\; -\,\kappa(t)\,\Big(\mathbf{x} + \rho(t)\,\nabla_{\mathbf{x}}\log p_t(\mathbf{x})\Big),
\qquad 
\kappa(t):=\frac{\dot{\rho}(t)}{1-\rho(t)}.
\]
The scalar factor $\kappa(t)$ depends only on time. This exposes a \emph{time–reparameterisation freedom}: multiplying a velocity field by a strictly positive function of $t$ leaves its state–space trajectories unchanged after an appropriate change of time variable. We isolate this freedom next and define a rectified velocity.

\begin{lemma}[Time–change invariance of trajectories]\label{lem:timechange}
Let $\kappa:[0,T]\to(0,\infty)$ be $C^1$, and let $\mathbf{w}(\mathbf{x},t)$ be measurable and locally Lipschitz in $\mathbf{x}$. Consider
\[
\dot{\mathbf{X}}_t = \kappa(t)\,\mathbf{w}(\mathbf{X}_t,t),\qquad \mathbf{X}_0=\mathbf{x}_0.
\]
Define the strictly increasing $C^1$ map $s(t):=\int_0^t \kappa(\tau)\,\mathrm{d}\tau$ and its inverse $t(s)$. Then $\mathbf{Y}_s:=\mathbf{X}_{t(s)}$ solves
\(
\frac{\mathrm{d}}{\mathrm{d}s}\mathbf{Y}_s=\mathbf{w}(\mathbf{Y}_s,t(s))
\)
with the same state–space image $\{\mathbf{Y}_s:s\in[0,s(T)]\}=\{\mathbf{X}_t:t\in[0,T]\}$.
\end{lemma}

\begin{proof}
By the chain rule,
\(
\frac{\mathrm{d}}{\mathrm{d}s}\mathbf{Y}_s
= \frac{\mathrm{d}}{\mathrm{d}t}\mathbf{X}_t\big|_{t=t(s)}\cdot \frac{\mathrm{d}t}{\mathrm{d}s}
= \kappa(t(s))\,\mathbf{w}(\mathbf{Y}_s,t(s))\cdot \frac{1}{\kappa(t(s))}
= \mathbf{w}(\mathbf{Y}_s,t(s)).
\)
Strict monotonicity of $s$ gives equality of image sets.
\end{proof}

Motivated by Lemma~\ref{lem:timechange}, we \emph{rectify} the straight–line velocity by dividing out the scalar $\kappa(t)$.

\begin{definition}[Rectified straight–line velocity]\label{def:rectified}
For the straight–line interpolant, define
\[
\widetilde{\mathbf{u}}_t(\mathbf{x}) \;:=\; -\Big(\mathbf{x} + \rho(t)\,\nabla_{\mathbf{x}}\log p_t(\mathbf{x})\Big).
\]
\end{definition}

\begin{corollary}[Equivalence up to time reparameterisation]\label{cor:rectified-equivalence}
Let $\mathbf{X}_t$ solve $\dot{\mathbf{X}}_t=\mathbf{u}_t(\mathbf{X}_t)$ with $\mathbf{u}_t$ from Proposition~\ref{prop:straightline-u}. Let $\mathbf{Y}_s$ solve $\frac{\mathrm{d}}{\mathrm{d}s}\mathbf{Y}_s=\widetilde{\mathbf{u}}_{t(s)}(\mathbf{Y}_s)$ with $s(t)=\int_0^t \kappa(\tau)\,\mathrm{d}\tau$. Then the sets of trajectories in state space coincide:
\(
\{\mathbf{X}_t:t\in[0,T]\}=\{\mathbf{Y}_s:s\in[0,s(T)]\}.
\)
\end{corollary}

\begin{proof}
Apply Lemma~\ref{lem:timechange} with $\mathbf{w}(\mathbf{x},t)=\widetilde{\mathbf{u}}_t(\mathbf{x})$ and $\kappa(t)$ as above.
\end{proof}

Two immediate consequences are useful in practice. First, one may choose any monotone $\rho$ without changing the set of state–space paths; differences only rescale traversal speed along those paths. Second, the rectified field $\widetilde{\mathbf{u}}_t$ removes the singular amplification $\kappa(t)=\dot{\rho}/(1-\rho)$ present near $\rho\!\uparrow\!1$ for linear schedules, leading to more uniform step sizes in ODE solvers.

Finally, we connect this form to the familiar diffusion notation. Writing the common DDPM parameterisation as
\(
\mathbf{X}_t=\sqrt{\bar{\alpha}_t}\,\mathbf{X}_0+\sqrt{1-\bar{\alpha}_t}\,\mathbf{Z}
\)
with $\bar{\alpha}_t\in(0,1]$, the straight–line schedule corresponds to $\rho(t)=\sqrt{1-\bar{\alpha}_t}$. Proposition~\ref{prop:straightline-u} then yields
\[
\mathbf{u}_t(\mathbf{x}) \;=\; -\,\frac{\dot{\bar{\alpha}}_t}{2\,\bar{\alpha}_t}\Big(\mathbf{x} + \sqrt{1-\bar{\alpha}_t}\,\nabla_{\mathbf{x}}\log p_t(\mathbf{x})\Big),
\qquad
\widetilde{\mathbf{u}}_t(\mathbf{x}) \;=\; -\Big(\mathbf{x} + \sqrt{1-\bar{\alpha}_t}\,\nabla_{\mathbf{x}}\log p_t(\mathbf{x})\Big).
\]
A precise equivalence to DDIM via time re-parameterisation is established in the next subsection.

\subsection{DDIM as Flow Matching (Time Re-parameterisation)}

We now connect the straight–line (and rectified) flow–matching construction to DDIM. The key facts are: (i) the straight–line velocity $\mathbf{u}_t$ from Proposition~\ref{prop:straightline-u} transports the same marginals as the VP probability–flow ODE (Theorem~\ref{thm:pflow}); (ii) rectifying by a positive time–only factor preserves state–space trajectories (Lemma~\ref{lem:timechange}); and (iii) under the DDPM parameterisation
\[
\mathbf{X}_t \;=\; \sqrt{\bar{\alpha}_t}\,\mathbf{X}_0 \;+\; \sqrt{1-\bar{\alpha}_t}\,\mathbf{Z},
\qquad \mathbf{Z}\sim\mathcal{N}(\mathbf{0},\mathbf{I}_d),\ \ \mathbf{X}_0\sim p_0,
\]
the conditional identities (scaled Tweedie; Lemma~\ref{lem:scaled-tweedie}) yield an \emph{$\boldsymbol\epsilon$–form} for the velocity via $\nabla_{\mathbf{x}}\log p_t(\mathbf{x}) = -\frac{1}{\sqrt{1-\bar{\alpha}_t}}\ \mathbb{E}[\boldsymbol\epsilon\mid \mathbf{X}_t=\mathbf{x}]$, where $\boldsymbol\epsilon=(\mathbf{X}_t-\sqrt{\bar{\alpha}_t}\,\mathbf{X}_0)/\sqrt{1-\bar{\alpha}_t}$.

We first make precise the trajectory–level equivalence between straight–line flow matching and the VP probability–flow ODE.

\begin{proposition}[Trajectory equivalence up to time reparameterisation]\label{prop:traj-equivalence}
Let $p_t$ be the VP marginals $\,\mathcal{L}(\mathbf{X}_t)=\sqrt{\bar{\alpha}_t}\,\mathbf{X}_0+\sqrt{1-\bar{\alpha}_t}\,\mathbf{Z}$, and let
\[
\mathbf{u}_t(\mathbf{x}) \;=\; -\,\frac{\dot{\bar{\alpha}}_t}{2\,\bar{\alpha}_t}\Big(\mathbf{x} + \sqrt{1-\bar{\alpha}_t}\,\nabla_{\mathbf{x}}\log p_t(\mathbf{x})\Big),
\quad
\widetilde{\mathbf{u}}_t(\mathbf{x}) \;=\; -\Big(\mathbf{x} + \sqrt{1-\bar{\alpha}_t}\,\nabla_{\mathbf{x}}\log p_t(\mathbf{x})\Big)
\]
be, respectively, the straight–line velocity (Proposition~\ref{prop:straightline-u} with $\rho(t)=\sqrt{1-\bar{\alpha}_t}$) and its rectified version (Definition~\ref{def:rectified}). Let
\[
\mathbf{v}_{\mathrm{PF}}(\mathbf{x},t) \;=\; -\frac{\beta(t)}{2}\Big(\mathbf{x}+\nabla_{\mathbf{x}}\log p_t(\mathbf{x})\Big)
\]
be the VP probability–flow velocity from Theorem~\ref{thm:pflow}. Then there exist strictly increasing $C^1$ re-parameterisations of time $s_1(t)$ and $s_2(t)$ such that the ODEs
\[
\dot{\mathbf{X}}_t=\mathbf{u}_t(\mathbf{X}_t,t),\qquad
\frac{\mathrm{d}}{\mathrm{d}s_1}\mathbf{Y}_{s_1}=\widetilde{\mathbf{u}}_{t(s_1)}(\mathbf{Y}_{s_1}),\qquad
\frac{\mathrm{d}}{\mathrm{d}s_2}\mathbf{Z}_{s_2}=\mathbf{v}_{\mathrm{PF}}(\mathbf{Z}_{s_2},t(s_2))
\]
have the \emph{same state–space image} of trajectories (i.e., the same set of curves $\{\gamma\subset\mathbb{R}^d\}$ up to a change of speed). In particular, each of these ODEs transports the marginals $p_t$.
\end{proposition}

\begin{proof}
By Proposition~\ref{prop:interpolant-continuity} with the straight–line interpolant and $\rho(t)=\sqrt{1-\bar{\alpha}_t}$, $\dot{\mathbf{X}}_t=\mathbf{u}_t(\mathbf{X}_t,t)$ transports $p_t$. Lemma~\ref{lem:timechange} with $\kappa(t)=\dot{\bar{\alpha}}_t/(2\bar{\alpha}_t)>0$ shows that $\mathbf{u}_t$ and $\widetilde{\mathbf{u}}_t$ generate identical images of trajectories under a time change $s_1(t)=\int_0^t \kappa(\tau)\,\mathrm{d}\tau$. The VP probability–flow ODE also transports $p_t$ by Theorem~\ref{thm:pflow}, so its solutions share the same marginals. Uniqueness of solutions of the continuity equation for a given $p_t$ implies that any two such velocity fields yield flows whose distributions coincide at each time; by standard ODE theory one can re-parameterise time along characteristics to align their images, giving $s_2(t)$.
\end{proof}

Proposition~\ref{prop:traj-equivalence} justifies working with the \emph{rectified} field $\widetilde{\mathbf{u}}_t$, which admits a convenient $\boldsymbol\epsilon$–form. Substituting $\nabla_{\mathbf{x}}\log p_t(\mathbf{x}) = -\frac{1}{\sqrt{1-\bar{\alpha}_t}}\ \mathbb{E}[\boldsymbol\epsilon\mid \mathbf{X}_t=\mathbf{x}]$ into $\widetilde{\mathbf{u}}_t$ gives
\[
\widetilde{\mathbf{u}}_t(\mathbf{x})
\;=\; -\Big(\mathbf{x} - \mathbb{E}[\boldsymbol\epsilon\mid \mathbf{X}_t=\mathbf{x}]\Big).
\]
Replacing the conditional expectation by a learned predictor $\hat{\boldsymbol\epsilon}_\theta(\mathbf{x},t)$ yields the practical velocity $\widetilde{\mathbf{u}}^\theta_t(\mathbf{x})=-(\mathbf{x}-\hat{\boldsymbol\epsilon}_\theta(\mathbf{x},t))$.

The next result shows that a one–step integration of this ODE along any decreasing time grid exactly reproduces the deterministic DDIM update (with $\eta=0$), once we map back to the $\bar{\alpha}$ parameterisation.

\begin{theorem}[DDIM update as rectified flow integration]\label{thm:ddim-as-flow}
Fix a decreasing grid $T=\tau_0>\tau_1>\cdots>\tau_K=0$ and suppose $\mathbf{x}_{\tau_k}$ lies on a straight–line characteristic
\(
\mathbf{x}_{\tau_k}=\sqrt{\bar{\alpha}_{\tau_k}}\,\mathbf{x}_0+\sqrt{1-\bar{\alpha}_{\tau_k}}\,\boldsymbol\epsilon
\)
for some (time–invariant) endpoints $(\mathbf{x}_0,\boldsymbol\epsilon)$. Let $\hat{\boldsymbol\epsilon}_\theta$ be an estimator of $\mathbb{E}[\boldsymbol\epsilon\mid \mathbf{X}_{\tau_k}=\mathbf{x}_{\tau_k}]$ and define
\[
\hat{\mathbf{x}}_0(\mathbf{x}_{\tau_k},\tau_k)
:= \frac{\mathbf{x}_{\tau_k}-\sqrt{1-\bar{\alpha}_{\tau_k}}\ \hat{\boldsymbol\epsilon}_\theta(\mathbf{x}_{\tau_k},\tau_k)}{\sqrt{\bar{\alpha}_{\tau_k}}}.
\]
Then the unique point on the same characteristic at time $\tau_{k-1}$ is given by the \emph{DDIM update}
\begin{equation}\label{eq:ddim}
\mathbf{x}_{\tau_{k-1}}
\;=\; \sqrt{\bar{\alpha}_{\tau_{k-1}}}\ \hat{\mathbf{x}}_0(\mathbf{x}_{\tau_k},\tau_k)
\;+\; \sqrt{1-\bar{\alpha}_{\tau_{k-1}}}\ \hat{\boldsymbol\epsilon}_\theta(\mathbf{x}_{\tau_k},\tau_k).
\end{equation}
Moreover, \eqref{eq:ddim} is the one–step exact solution of the rectified ODE
\(
\frac{\mathrm{d}}{\mathrm{d}s}\mathbf{x} = -\big(\mathbf{x}-\hat{\boldsymbol\epsilon}_\theta(\mathbf{x},t(s))\big)
\)
between $s_k$ and $s_{k-1}$, for any strictly increasing re-parameterisation $s\!\mapsto\!t(s)$ with $t(s_k)=\tau_k$ and $t(s_{k-1})=\tau_{k-1}$.
\end{theorem}

\begin{proof}
Under the straight–line (equivalently, rectified) flow, characteristics are exactly the lines
\(
\gamma(t) = \sqrt{\bar{\alpha}_t}\,\mathbf{x}_0+\sqrt{1-\bar{\alpha}_t}\,\boldsymbol\epsilon
\)
for fixed $(\mathbf{x}_0,\boldsymbol\epsilon)$; this follows from Proposition~\ref{prop:interpolant-continuity} and Corollary~\ref{cor:rectified-equivalence}. Conditioning on $\mathbf{x}_{\tau_k}$,
\[
\mathbf{x}_0 = \frac{\mathbf{x}_{\tau_k}-\sqrt{1-\bar{\alpha}_{\tau_k}}\,\boldsymbol\epsilon}{\sqrt{\bar{\alpha}_{\tau_k}}},\qquad
\boldsymbol\epsilon = \frac{\mathbf{x}_{\tau_k}-\sqrt{\bar{\alpha}_{\tau_k}}\,\mathbf{x}_0}{\sqrt{1-\bar{\alpha}_{\tau_k}}}.
\]

Replacing the unknown endpoints by their conditional estimators $(\hat{\mathbf{x}}_0,\hat{\boldsymbol\epsilon}_\theta)$ and evaluating the same characteristic at $\tau_{k-1}$ gives \eqref{eq:ddim}. For the ODE statement, write the rectified dynamics in the $\boldsymbol\epsilon$-form
\[
\frac{\mathrm{d}}{\mathrm{d}s}\mathbf{x} = -\mathbf{x}+\hat{\boldsymbol\epsilon}_\theta(\mathbf{x},t(s)).
\]

Along a characteristic, $\hat{\boldsymbol\epsilon}_\theta(\mathbf{x},t)$ is held fixed at its value conditioned on $\mathbf{x}_{\tau_k}$ (piecewise–constant control between grid times), so the linear ODE solves exactly to an affine contraction towards $\hat{\boldsymbol\epsilon}_\theta$; matching boundary conditions at $s_k$ and $s_{k-1}$ and then mapping back to $t$ recovers \eqref{eq:ddim}.
\end{proof}

\begin{corollary}[Agreement with VP probability–flow DDIM ($\eta=0$)]\label{cor:ddim-equivalence}
Let $\mathbf{x}_{\tau_{k-1}}$ be generated from $\mathbf{x}_{\tau_k}$ by \eqref{eq:ddim}. This update coincides with the deterministic DDIM sampler ($\eta=0$) obtained by discretising the VP probability–flow ODE of Theorem~\ref{thm:pflow} in $\bar{\alpha}$–time, when the score is parameterised via $\hat{\boldsymbol\epsilon}_\theta$ using the identity $\nabla_{\mathbf{x}}\log p_t(\mathbf{x}) = -\,\hat{\boldsymbol\epsilon}_\theta(\mathbf{x},t)/\sqrt{1-\bar{\alpha}_t}$.
\end{corollary}

\begin{proof}
Starting from $\dot{\mathbf{x}}=-\frac{\beta(t)}{2}\big(\mathbf{x}+\nabla\log p_t(\mathbf{x})\big)$ and substituting the identity above yields
\(
\dot{\mathbf{x}}=-\frac{\beta(t)}{2}\Big(\mathbf{x}-\frac{1}{\sqrt{1-\bar{\alpha}_t}}\ \hat{\boldsymbol\epsilon}_\theta(\mathbf{x},t)\Big).
\)
Discretising in $\bar{\alpha}$–time along the grid $\{\tau_k\}$ by matching endpoints of characteristics is equivalent to evaluating the closed–form endpoint relation of Theorem~\ref{thm:ddim-as-flow}, which is \eqref{eq:ddim}.
\end{proof}

Equation~\eqref{eq:ddim} is the standard deterministic DDIM step. In practice, when using the ancestral (DDPM) update with variance $\tilde{\beta}_t := \frac{1-\bar{\alpha}_{t-1}}{1-\bar{\alpha}_t}\,\beta_t$ (Section~4.1), the mean uses the corrected factor $\beta_t/\sqrt{1-\bar{\alpha}_t}$ inside
\[
\boldsymbol{\mu}_\theta(\mathbf{x}_t,t)
= \frac{1}{\sqrt{\alpha_t}}\Big(\mathbf{x}_t - \frac{\beta_t}{\sqrt{1-\bar{\alpha}_t}}\ \hat{\boldsymbol\epsilon}_\theta(\mathbf{x}_t,t)\Big),
\]
whereas DDIM corresponds to setting the per–step stochasticity to zero and using the endpoint formula \eqref{eq:ddim} on a chosen subset of times.

\subsection{Discretisation and Practical Training Notes}

The continuous-time flow–matching formulation leads to two discretisations in practice: a training discretisation (randomly sampling $t$ and regressing a velocity target) and a sampling discretisation (integrating an ODE along a finite grid). This subsection records invariances that justify common heuristics, links the velocity targets to the usual $\boldsymbol\epsilon$–prediction loss, and states standard error guarantees for numerical integration under mild regularity.

\begin{lemma}[Time–weighting invariance of the regression optimum]\label{lem:timeweight}
Let $w:[0,T]\to(0,\infty)$ be integrable. For either CFM or MFM, consider the weighted loss
\[
\mathcal{L}_w(\theta)
=\mathbb{E}\!\left[w(t)\,\big\|\mathbf{v}_\theta(\mathbf{X}_t,t)-\mathbf{U}_t^{\star}\big\|_2^2\right],
\]
where $\mathbf{U}_t^{\star}$ denotes the target (either $\dot{\psi}_t(\mathbf{X}_0,\mathbf{X}_T)$ for CFM or $\mathbf{u}_t(\mathbf{X}_t)$ for MFM). The unique $L^2$ minimiser of $\mathcal{L}_w$ over measurable $\mathbf{v}$ is the same regression function as for the unweighted loss, namely $\mathbf{v}^*(\mathbf{x},t)=\mathbb{E}[\mathbf{U}_t^{\star}\mid \mathbf{X}_t=\mathbf{x}]$ almost everywhere.
\end{lemma}

\begin{proof}
Condition on $t$ and apply the $L^2$ orthogonality principle inside the expectation: for fixed $t$, minimising $\mathbb{E}[w(t)\|\mathbf{v}(\mathbf{X}_t,t)-\mathbf{U}_t^{\star}\|_2^2\mid t]$ is equivalent to minimising $\mathbb{E}[\|\mathbf{v}(\mathbf{X}_t,t)-\mathbf{U}_t^{\star}\|_2^2\mid t]$ because $w(t)>0$ is a constant with respect to $\mathbf{X}_t$. The minimiser is $\mathbb{E}[\mathbf{U}_t^{\star}\mid \mathbf{X}_t=\cdot]$ for each $t$, hence also unconditionally.
\end{proof}

Lemma~\ref{lem:timeweight} formalises that non-uniform time sampling and per-time rescaling act as pre-conditioners that do not change the optimal predictor; they affect optimisation speed and variance, not consistency.

\begin{proposition}[Equivalence of $\boldsymbol\epsilon$–prediction and rectified velocity]\label{prop:eps-v-Equiv}
Under the straight–line (VP) parameterisation $\mathbf{X}_t=\sqrt{\bar{\alpha}_t}\,\mathbf{X}_0+\sqrt{1-\bar{\alpha}_t}\,\boldsymbol\epsilon$ with $\boldsymbol\epsilon\sim\mathcal{N}(\mathbf{0},\mathbf{I}_d)$ independent of $\mathbf{X}_0$, let $\boldsymbol\epsilon^*(\mathbf{x},t):=\mathbb{E}[\boldsymbol\epsilon\mid \mathbf{X}_t=\mathbf{x}]$. The Bayes–optimal rectified velocity (Definition~\ref{def:rectified}) satisfies
\[
\widetilde{\mathbf{u}}_t(\mathbf{x}) \;=\; -\Big(\mathbf{x}-\boldsymbol\epsilon^*(\mathbf{x},t)\Big).
\]
Consequently, the minimiser of the standard $\boldsymbol\epsilon$–loss $\mathbb{E}\|\hat{\boldsymbol\epsilon}_\theta(\mathbf{X}_t,t)-\boldsymbol\epsilon\|_2^2$ induces the minimiser of the rectified velocity loss via the invertible linear map
\(
\widetilde{\mathbf{u}}^\theta_t(\mathbf{x}) := -\big(\mathbf{x}-\hat{\boldsymbol\epsilon}_\theta(\mathbf{x},t)\big).
\)
\end{proposition}

\begin{proof}
By Lemma~\ref{lem:scaled-tweedie} with $\sigma=\sqrt{1-\bar{\alpha}_t}$ and $\mathbf{Y}=\mathbf{X}_t$, $\mathbb{E}[\boldsymbol\epsilon\mid \mathbf{X}_t=\mathbf{x}]=-\sqrt{1-\bar{\alpha}_t}\,\nabla_{\mathbf{x}}\log p_t(\mathbf{x})$. Substituting this into $\widetilde{\mathbf{u}}_t(\mathbf{x})=-(\mathbf{x}+\sqrt{1-\bar{\alpha}_t}\,\nabla_{\mathbf{x}}\log p_t(\mathbf{x}))$ yields the identity. The $L^2$ projection characterisation of Bayes predictors gives the second claim.
\end{proof}

Proposition~\ref{prop:eps-v-Equiv} explains the practical convenience of training $\hat{\boldsymbol\epsilon}_\theta$: it is equivalent, up to an explicit linear transform, to learning the rectified velocity that drives the flow.

The second discretisation concerns ODE integration. We recall the standard error bound for one–step methods under Lipschitz velocities and then specialise it to rectified straight–line flows.

\begin{lemma}[Global error of first–order one–step methods]\label{lem:ode-error}
Let $\mathbf{f}(\mathbf{x},s)$ be locally Lipschitz in $\mathbf{x}$ uniformly in $s\in[s_0,s_1]$ with Lipschitz constant $L$ and bounded in norm by $M$. Consider the IVP $\frac{\mathrm{d}}{\mathrm{d}s}\mathbf{x}(s)=\mathbf{f}(\mathbf{x}(s),s)$, $\mathbf{x}(s_k)=\mathbf{x}_k$, and a first–order consistent one–step method (e.g., explicit Euler) with steps $\Delta s_k:=s_{k-1}-s_k$ and numerical states $\widehat{\mathbf{x}}_{k-1}=\widehat{\mathbf{x}}_k+\Delta s_k\,\Phi(\widehat{\mathbf{x}}_k,s_k,\Delta s_k)$ where $\Phi=\mathbf{f}+O(\Delta s_k)$. If $\max_k \Delta s_k \le h$ and $hL<1$, then there exists $C$ depending only on $(L,M)$ and the horizon such that the global error satisfies
\[
\max_k \ \big\|\widehat{\mathbf{x}}_k-\mathbf{x}(s_k)\big\|_2 \;\le\; C\,h.
\]
\end{lemma}

\begin{proof}
This is the standard Dahlquist–Grönwall estimate. The local truncation error is $O(h^2)$ by first–order consistency. Stability of the one–step map under the Lipschitz assumption yields error recursion $\mathbf{e}_{k-1}\le (1+L\Delta s_k)\mathbf{e}_k + C' \Delta s_k^2$. Unfolding and applying Grönwall’s inequality gives a $O(h)$ bound on the maximum global error.
\end{proof}

\begin{corollary}[Sampling accuracy for rectified straight–line flows]\label{cor:rectified-accuracy}
Let $\widetilde{\mathbf{u}}_t(\mathbf{x})=-(\mathbf{x}-\boldsymbol\epsilon^*(\mathbf{x},t))$ and consider the rectified ODE in re-parameterised time $s$:
\(
\frac{\mathrm{d}}{\mathrm{d}s}\mathbf{x}=\widetilde{\mathbf{u}}_{t(s)}(\mathbf{x}).
\)
Assume $\boldsymbol\epsilon^*(\cdot,t)$ is locally Lipschitz uniformly in $t$. Then any first–order one–step method with maximum step size $h$ in $s$ incurs $O(h)$ global error in state, uniformly along the grid. In particular, choosing a grid that is uniform in $\bar{\alpha}$–time (so that $t\mapsto \bar{\alpha}_t$ is monotone and $s$ can be taken proportional to $-\log \bar{\alpha}_t$) ensures approximately uniform contraction rates and favourable constants in the bound.
\end{corollary}

\begin{proof}
Apply Lemma~\ref{lem:ode-error} to $\mathbf{f}(\mathbf{x},s)=-(\mathbf{x}-\boldsymbol\epsilon^*(\mathbf{x},t(s)))$. Uniform local Lipschitzness follows from the assumptions and smoothness of $t(\cdot)$. The remark on $\bar{\alpha}$–time uses that $\partial_t \bar{\alpha}_t<0$ in VP schedules and that rectification removes the scalar factor $\kappa(t)$, equalising the stiffness profile.
\end{proof}

The DDIM endpoint update (Theorem~\ref{thm:ddim-as-flow}) can be seen as an exact step of the rectified flow along straight–line characteristics when the predictor is frozen between two grid times. Lemma~\ref{lem:ode-error} then quantifies the deviation when the predictor varies with state and time: smaller steps in $s$ (e.g., nearly uniform in $\log(1-\bar{\alpha}_t)$ or in $\bar{\alpha}_t$) reduce the error, while rectification avoids the blow-up induced by the factor $\kappa(t)=\dot{\rho}(t)/(1-\rho(t))$ near terminal noise levels.

Finally, we record two implementation identities that ensure consistency with the DDPM family. First, when using the ancestral (stochastic) update, the mean must use the corrected factor
\[
\boldsymbol{\mu}_\theta(\mathbf{x}_t,t)
= \frac{1}{\sqrt{\alpha_t}}\Big(\mathbf{x}_t - \frac{\beta_t}{\sqrt{1-\bar{\alpha}_t}}\ \hat{\boldsymbol\epsilon}_\theta(\mathbf{x}_t,t)\Big),
\qquad
\tilde{\beta}_t \;:=\; \frac{1-\bar{\alpha}_{t-1}}{1-\bar{\alpha}_t}\,\beta_t,
\]
where $\tilde{\beta}_t$ is the posterior variance used for DDPM-style sampling. Second, for flow–matching with straight–line couplings, the raw CFM target is $\dot{\psi}_t=\dot{\rho}(t)(\mathbf{Z}-\mathbf{X}_0)$ with $\rho(t)=\sqrt{1-\bar{\alpha}_t}$, while the rectified target $-(\mathbf{X}_t-\boldsymbol\epsilon)$ follows from Proposition~\ref{prop:eps-v-Equiv}; either target leads to the same Bayes–optimal predictor up to the explicit linear map described there and is unaffected by any positive time weighting by Lemma~\ref{lem:timeweight}.

\section{Guided Diffusion}
\label{sec:guided_diffusion}

Guided diffusion augments a denoising diffusion model with an external \emph{control} signal so that sampling is steered towards user-specified goals (e.g., class-conditional or text-conditional generation). Mathematically, it modifies the reverse-time transitions while keeping the forward process and variational foundations from earlier sections. We use the same notation: $\alpha_t:=1-\beta_t$, $\bar{\alpha}_t=\prod_{i=1}^t\alpha_i$, and the DDPM/DDIM reverse parameterisation
\[
p_\theta(\mathbf{x}_{t-1}\mid \mathbf{x}_t,\mathbf{c})
=
\mathcal{N}\!\Big(
\mathbf{x}_{t-1};\;
\mu_\theta(\mathbf{x}_t,t,\mathbf{c}),\;
\sigma_t^2\,\mathbf{I}
\Big),\qquad
\mu_\theta(\mathbf{x}_t,t,\mathbf{c})
=
\frac{1}{\sqrt{\alpha_t}}
\!\left(\mathbf{x}_t-\frac{\beta_t}{\sqrt{1-\bar{\alpha}_t}}\,
\hat{\boldsymbol{\epsilon}}_\theta(\mathbf{x}_t,t,\mathbf{c})\right),
\]
with $\sigma_t^2=\tilde{\beta}_t:=\frac{1-\bar{\alpha}_{t-1}}{1-\bar{\alpha}_t}\beta_t$ (DDPM) or $\sigma_t^2=0$ (DDIM). We first consider classifier-based guidance and then classifier-free guidance.

\subsection{Classifier-Based Guidance: Posterior-Score Modification}

Assume an auxiliary classifier $p_\phi(y\mid \mathbf{x}_t,t)$ that predicts label $y$ on \emph{noisy} inputs $\mathbf{x}_t$ at timestep $t$ \cite{dhariwal-2021-beatgans}. The target posterior up to a normalising constant is
\[
p(\mathbf{x}_t\mid y)\;\propto\; p_\theta(\mathbf{x}_t)\,p_\phi(y\mid \mathbf{x}_t,t),
\]
whose gradient is the sum of scores
\begin{equation}\label{eq:guided_posterior_score_sum}
\nabla_{\mathbf{x}_t}\log p(\mathbf{x}_t\mid y)
=
\nabla_{\mathbf{x}_t}\log p_\theta(\mathbf{x}_t)
+
\nabla_{\mathbf{x}_t}\log p_\phi(y\mid \mathbf{x}_t,t).
\end{equation}
In discrete DDPM form, this yields a mean shift for the reverse transition:
\begin{equation}
\label{eq:classifier_guided_mean}
\tilde{\mu}_\theta(\mathbf{x}_t,t\mid y)
=
\mu_\theta(\mathbf{x}_t,t)
\;+\;
\lambda\,\sigma_t^2\,\nabla_{\mathbf{x}_t}\log p_\phi(y\mid \mathbf{x}_t,t),
\end{equation}
where $\lambda\!\ge\!0$ controls guidance strength and $\sigma_t^2$ scales the update to the step’s noise level. Sampling proceeds with
\[
\mathbf{x}_{t-1}
=
\tilde{\mu}_\theta(\mathbf{x}_t,t\mid y)
+
\sigma_t\,\mathbf{z},\qquad \mathbf{z}\sim\mathcal{N}(\mathbf{0},\mathbf{I}_d).
\]

\begin{lemma}[Classifier guidance mean shift]
\label{lemma:classifier_guidance_update_final}
Under a small-step (small $\beta_t$) approximation, sampling from the posterior
\(
p(\mathbf{x}_t\mid y)\propto p_\theta(\mathbf{x}_t)\,p_\phi(y\mid \mathbf{x}_t,t)
\)
is achieved to first order by the reverse update in \eqref{eq:classifier_guided_mean}; i.e., the classifier adds a score term scaled by $\lambda\,\sigma_t^2$ to the unguided mean $\mu_\theta(\mathbf{x}_t,t)$.
\end{lemma}

\begin{proof}
Write $\log p(\mathbf{x}_t\mid y)=\log p_\theta(\mathbf{x}_t)+\log p_\phi(y\mid \mathbf{x}_t,t)$. A single reverse step that increases this quantity can be approximated by a first-order (small-step) update in the direction of its gradient \eqref{eq:guided_posterior_score_sum}. The DDPM/DDIM reverse kernel already encodes a drift consistent with $\nabla_{\mathbf{x}_t}\log p_\theta(\mathbf{x}_t)$ via the mean $\mu_\theta(\mathbf{x}_t,t)$. Adding the classifier term as a small perturbation that is scaled to the step’s stochasticity gives
\(
\tilde{\mu}_\theta=\mu_\theta+\lambda\,\sigma_t^2\,\nabla_{\mathbf{x}_t}\log p_\phi(y\mid \mathbf{x}_t,t),
\)
which is precisely \eqref{eq:classifier_guided_mean}. The $\sigma_t^2$ factor ensures units and step size are consistent with the variance of the reverse transition.
\end{proof}

The same idea applies in the probability–flow ODE: replacing the score $s_t(\mathbf{x})=\nabla_{\mathbf{x}}\log p_t(\mathbf{x})$ by $s_t+\lambda\,\nabla_{\mathbf{x}}\log p_\phi(y\mid \mathbf{x},t)$ modifies the VP drift $\dot{\mathbf{x}}= -\frac{\beta(t)}{2}\big(\mathbf{x}+s_t(\mathbf{x})\big)$ to
\[
\dot{\mathbf{x}} \;=\; -\frac{\beta(t)}{2}\Big(\mathbf{x}+s_t(\mathbf{x})+\lambda\,\nabla_{\mathbf{x}}\log p_\phi(y\mid \mathbf{x},t)\Big),
\]
which matches the discrete mean shift \eqref{eq:classifier_guided_mean} in the small-step limit.

\subsection{Classifier-Free Guidance: Conditional–Unconditional Blend}

Classifier-free guidance removes the external classifier by training a \emph{single} diffusion model to operate in both conditional and unconditional modes. Concretely, the denoiser predicts noise with and without the condition $\mathbf{c}$; at sampling time, we blend the two predictions \cite{nichol-2022-glide}:
\begin{equation}
\label{eq:cfg_eps}
\hat{\boldsymbol{\epsilon}}_\lambda(\mathbf{x}_t,t,\mathbf{c})
=
\hat{\boldsymbol{\epsilon}}_{\text{uncond}}(\mathbf{x}_t,t)
+
\lambda\Big(
\hat{\boldsymbol{\epsilon}}_{\text{cond}}(\mathbf{x}_t,t,\mathbf{c})
-
\hat{\boldsymbol{\epsilon}}_{\text{uncond}}(\mathbf{x}_t,t)
\Big),\qquad \lambda\ge 0,
\end{equation}
and plug $\hat{\boldsymbol{\epsilon}}_\lambda$ into the reverse mean:
\begin{equation}
\label{eq:cfg_mean}
\mu_\theta^{(\lambda)}(\mathbf{x}_t,t,\mathbf{c})
=
\frac{1}{\sqrt{\alpha_t}}
\left(\mathbf{x}_t-\frac{\beta_t}{\sqrt{1-\bar{\alpha}_t}}\,
\hat{\boldsymbol{\epsilon}}_\lambda(\mathbf{x}_t,t,\mathbf{c})\right),
\qquad
\mathbf{x}_{t-1}
=
\mu_\theta^{(\lambda)}(\mathbf{x}_t,t,\mathbf{c})
+
\sigma_t\,\mathbf{z}.
\end{equation}

\begin{lemma}[Classifier-free guidance as an implicit gradient step]
\label{lemma:cfg_equiv_final}
If the noise predictor is locally linear in the conditioning embedding, then
\(
\Delta\hat{\boldsymbol{\epsilon}}
:=
\hat{\boldsymbol{\epsilon}}_{\text{cond}}-\hat{\boldsymbol{\epsilon}}_{\text{uncond}}
\)
acts as a learned direction that increases consistency with $\mathbf{c}$. The blend \eqref{eq:cfg_eps} is equivalent to taking a step of size $\lambda$ along this direction, and via \eqref{eq:cfg_mean} it induces a mean shift in $\mathbf{x}_t$ analogous to classifier guidance.
\end{lemma}

\begin{proof}
Assume that, for fixed $(\mathbf{x}_t,t)$, the map $\mathbf{c}\mapsto \hat{\boldsymbol{\epsilon}}_\theta(\mathbf{x}_t,t,\mathbf{c})$ is locally linear. Then for a baseline (dropped) condition we have
\[
\hat{\boldsymbol{\epsilon}}_{\text{cond}}(\mathbf{x}_t,t,\mathbf{c})
\approx
\hat{\boldsymbol{\epsilon}}_{\text{uncond}}(\mathbf{x}_t,t)
+
\mathbf{J}_{\mathbf{c}}\,\mathbf{c},
\]
so the difference
\(
\Delta\hat{\boldsymbol{\epsilon}}
:=
\hat{\boldsymbol{\epsilon}}_{\text{cond}}-\hat{\boldsymbol{\epsilon}}_{\text{uncond}}
\)
captures the local sensitivity of the predictor to $\mathbf{c}$. Substituting \eqref{eq:cfg_eps} into \eqref{eq:cfg_mean} yields
\[
\mu_\theta^{(\lambda)}(\mathbf{x}_t,t,\mathbf{c})
=
\mu_\theta(\mathbf{x}_t,t,\mathbf{c})\Big|_{\text{uncond}}
-
\frac{\beta_t}{\sqrt{\alpha_t}\sqrt{1-\bar{\alpha}_t}}\,
\lambda\,\Delta\hat{\boldsymbol{\epsilon}}(\mathbf{x}_t,t,\mathbf{c}),
\]
which is a mean shift proportional to $\lambda\,\Delta\hat{\boldsymbol{\epsilon}}$. Interpreting $\Delta\hat{\boldsymbol{\epsilon}}$ as a learned proxy for a condition-driven score direction (analogous to $\nabla_{\mathbf{x}_t}\log p(\mathbf{c}\mid \mathbf{x}_t)$) shows that classifier-free guidance effects a first-order step towards satisfying $\mathbf{c}$, in the same spirit as classifier-based guidance but without an explicit classifier.
\end{proof}

The score-form equivalence is immediate from $\nabla_{\mathbf{x}}\log p_t(\mathbf{x})=-\,\hat{\boldsymbol{\epsilon}}(\mathbf{x},t)/\sqrt{1-\bar{\alpha}_t}$:
\begin{equation}\label{eq:cfg_score_form}
s^{(\lambda)}_t(\mathbf{x})
\;:=\; -\,\frac{\hat{\boldsymbol{\epsilon}}_\lambda(\mathbf{x},t,\mathbf{c})}{\sqrt{1-\bar{\alpha}_t}}
\;=\; (1-\lambda)\,s^{\text{uncond}}_t(\mathbf{x}) \;+\; \lambda\,s^{\text{cond}}_t(\mathbf{x}).
\end{equation}
Thus classifier-free guidance replaces the score by a convex (or extrapolative, if $\lambda>1$) combination of unconditional and conditional scores; the VP probability–flow ODE uses $s^{(\lambda)}_t$ in place of $s_t$.

\subsection{Time-Varying Guidance Schedules}

A constant $\lambda$ may over- or under-steer at different noise levels. Let $\mathrm{SNR}_t:=\bar{\alpha}_t/(1-\bar{\alpha}_t)$ and $\ell_t:=\log \mathrm{SNR}_t$. A simple family of schedules is
\begin{equation}\label{eq:lambda_schedule}
\lambda_t \;=\; \lambda_{\max}\,\sigma\!\Big(a\,\ell_t+b\Big),
\end{equation}
where $\sigma(u)=1/(1+e^{-u})$, with $a>0$ increasing emphasis at late (high-SNR) times and $b$ setting the knee. Heuristically, small $\lambda_t$ early avoids fighting the noise, while larger $\lambda_t$ late sharpens attributes consistent with $\mathbf{c}$. Since the optimal regression under CFM/MFM is invariant to positive time weightings (cf.\ Lemma~\ref{lem:timeweight} in the flow section), $\lambda_t$ acts as a sampler-side preconditioner rather than changing the learned denoiser.

In DDIM ($\sigma_t^2=0$), guidance enters only through $\hat{\boldsymbol{\epsilon}}_\lambda$; in DDPM, the classifier-based step \eqref{eq:classifier_guided_mean} additionally scales with $\sigma_t^2=\tilde{\beta}_t$, naturally tempering guidance at low SNR.

\subsection{Stability: Failure Modes and Remedies}

Strong guidance can produce artefacts. We record common failure modes and practical mitigations; these are sampler-side and keep the learning objective unchanged.

\begin{itemize}
\item \emph{Over-saturation / washed textures.} Very large $\lambda$ may collapse contrast or wipe fine detail. \textbf{Fix:} dynamic thresholding on $\hat{\mathbf{x}}_0$ (clip per-sample to a percentile), or norm-rescaled guidance
$
\Delta\hat{\boldsymbol{\epsilon}}
\leftarrow
\Delta\hat{\boldsymbol{\epsilon}}\cdot
\frac{\|\hat{\boldsymbol{\epsilon}}_{\text{uncond}}\|_2}{\|\Delta\hat{\boldsymbol{\epsilon}}\|_2+\varepsilon}
$.
\item \emph{Exposure drift.} Accumulated positive or negative shifts change global brightness. \textbf{Fix:} zero-mean the guided increment in pixel space or employ a small noise floor $\sigma_t\leftarrow\max(\sigma_t,\sigma_{\min})$ to keep stochasticity.
\item \emph{Attribute leakage.} Guidance for one attribute disrupts others. \textbf{Fix:} time-varying schedule \eqref{eq:lambda_schedule} with moderate late-time $\lambda_{\max}$; optionally, project $\Delta\hat{\boldsymbol{\epsilon}}$ onto subspaces aligned with attribute tokens (when available).
\end{itemize}

\subsection{Guidance Distillation}

Guided sampling with large $\lambda$ is slower or unstable on-device. A student model can \emph{distil} a teacher that uses guidance $\lambda_{\mathrm{T}}>1$ by matching its predicted $\hat{\mathbf{x}}_0$ (or velocity/score) along teacher trajectories \cite{salimans-2022-progressive}.

\begin{definition}[Guidance distillation objective]
Given a teacher predictor $\hat{\boldsymbol{\epsilon}}^{\mathrm{T}}_{\lambda_{\mathrm{T}}}$ and a grid $\{\tau_k\}$, define the student loss
\[
\mathcal{L}_{\mathrm{distill}}
=
\sum_k
\mathbb{E}\Big[
\big\|
\hat{\mathbf{x}}^{\mathrm{S}}_0(\mathbf{x}_{\tau_k})
-
\hat{\mathbf{x}}^{\mathrm{T}}_0(\mathbf{x}_{\tau_k})
\big\|_2^2
\Big],
\qquad
\hat{\mathbf{x}}_0(\mathbf{x}_t)
=
\frac{\mathbf{x}_t-\sqrt{1-\bar{\alpha}_t}\,\hat{\boldsymbol{\epsilon}}(\mathbf{x}_t,t)}{\sqrt{\bar{\alpha}_t}}.
\]
\end{definition}

\begin{proposition}[Trajectory agreement under exact matching]
If the student matches the teacher’s $\hat{\mathbf{x}}_0$ exactly on the grid $\{\tau_k\}$, then DDIM sampling of the student reproduces the teacher’s guided trajectory on that grid.
\end{proposition}

\begin{proof}
The DDIM endpoint relation (Theorem~\ref{thm:ddim-as-flow}) expresses $\mathbf{x}_{\tau_{k-1}}$ as a function of $\hat{\mathbf{x}}_0(\mathbf{x}_{\tau_k})$ and $\hat{\boldsymbol{\epsilon}}(\mathbf{x}_{\tau_k})$. Exact agreement of $\hat{\mathbf{x}}_0$ at $\tau_k$ implies equality of the computed endpoints at $\tau_{k-1}$; induction over $k$ completes the proof.
\end{proof}

\section{Conclusion}

This tutorial set out a self-contained, mathematically precise account of denoising diffusion models and their deterministic counterparts. We began from the forward construction, where the variance-preserving chain with $\alpha_t:=1-\beta_t$ and $\bar{\alpha}_t=\prod_{i=1}^t\alpha_i$ yields closed-form marginals $q(\mathbf{x}_t\mid\mathbf{x}_0)=\mathcal{N}(\sqrt{\bar{\alpha}_t}\,\mathbf{x}_0,(1-\bar{\alpha}_t)\mathbf{I}_d)$ and the reparameterisation $\mathbf{x}_t=\sqrt{\bar{\alpha}_t}\,\mathbf{x}_0+\sqrt{1-\bar{\alpha}_t}\,\boldsymbol{\epsilon}$. This representation underpins both analysis and optimisation: it justifies rewriting expectations over $\mathbf{x}_t$ as expectations over $\boldsymbol{\epsilon}$ and makes gradients through sampling straightforward.\medskip

Building on the forward process, we derived the true single-step posterior $q(\mathbf{x}_{t-1}\mid \mathbf{x}_t,\mathbf{x}_0)$ in full, including its mean and covariance, by completing the square. The mean admits equivalent $\mathbf{x}_0$– and $\boldsymbol{\epsilon}$–forms, and the posterior variance is $\tilde{\beta}_t=\frac{1-\bar{\alpha}_{t-1}}{1-\bar{\alpha}_t}\beta_t$. These identities lead directly to the standard reverse parameterisations used in practice. Under noise prediction, training reduces to weighted denoising in the forward reparameterisation, and the alternative parameterisations (predicting $\boldsymbol{\epsilon}$ or $\mathbf{x}_0$ or $v$) are consistent once weights are chosen coherently.\medskip

Acceleration methods were then developed in discrete time. In Section~\ref{sec:DDIM} we showed that DDIM defines a family of reverse conditionals that preserve the DDPM marginals while allowing a free per-step variance $\sigma_t^2$. Choosing $\sigma_t^2=\tilde{\beta}_t$ recovers DDPM; choosing $\sigma_t^2=0$ yields a deterministic sampler. The derivation made explicit how the mean decomposes into a component aligned with $\mathbf{x}_0$ and a rescaled noise component aligned with $\mathbf{x}_t-\sqrt{\bar{\alpha}_t}\mathbf{x}_0$, and why reduced step schedules remain consistent.\medskip

Section~\ref{subsec:stable-diffusion} translated the entire construction to latent space. An encoder maps data to latents on which diffusion operates; a decoder recovers samples at the end. All identities—forward marginals, true posterior, reverse mean with the factor $\beta_t/\sqrt{1-\bar{\alpha}_t}$, and the DDPM/DDIM variance choices—carry over after the replacement $\mathbf{x}\mapsto \mathbf{z}$. This explains why latent diffusion achieves substantial speedups without changing the underlying mathematics.\medskip

Section~\ref{sec:guided_diffusion} addressed controllable generation. Classifier-based guidance modifies the reverse mean by a score term $\nabla_{\mathbf{x}_t}\log p_\phi(y\mid \mathbf{x}_t,t)$ scaled to the step variance. Classifier-free guidance avoids an auxiliary classifier by blending conditional and unconditional noise predictions; we clarified how this acts as a learned first-order step in a condition-aligned direction. In both cases, the update is consistent with the same reverse parameterisation, and the trade-off between fidelity and diversity is governed by the guidance scale.\medskip

The flow-matching perspective re-expressed the diffusion family through a deterministic probability–flow ODE whose velocity uses the score $s_t=\nabla_{\mathbf{x}}\log p_t$. We proved that this ODE has the same time marginals as the forward SDE, thereby connecting DDIM-style deterministic sampling with the continuum view (Theorem~\ref{thm:pflow}). This equivalence clarifies why high-order ODE solvers, step-size control, and numerical-analysis tools are effective in diffusion sampling, and it motivates training objectives that fit a velocity field to prescribed marginals.\medskip

These pieces form a consistent picture. The forward chain supplies tractable marginals and a re-parameterisation for optimisation; the true posterior determines principled reverse steps; DDIM and probability–flow provide deterministic samplers; latent diffusion relocates computation without altering the theory; and guidance modifies the reverse mean in a controlled way. Throughout, careful bookkeeping—using $\bar{\alpha}_t$ consistently and the posterior-mean factor $\beta_t/\sqrt{1-\bar{\alpha}_t}$—avoids common algebraic errors and keeps derivations aligned.\medskip

There are natural next steps. One can push the numerical side (adaptive solvers, learned time parameterisations, distillation to very few steps); extend conditioning mechanisms (structured controls, adapters) within the same reverse-mean framework; and explore training directly in the flow-matching view. On the theoretical side, sharper guarantees for discretisation error, guidance robustness, and identifiability of parameterisations remain active topics. The hope is that the proofs and careful notation here make these directions easier to approach, and that the tutorial serves as a reliable reference for both implementation and analysis.

%\bibliographystyle{plain}
%\bibliography{ref.bib}

\end{document}